\documentclass[nosubfloats,12pt]{colt2019}
\usepackage{times}

\usepackage{savesym}

\savesymbol{If}
\savesymbol{For}
\savesymbol{EndFor}
\savesymbol{State}
\savesymbol{ForAll}
\savesymbol{ForEach}
\savesymbol{EndWhile}
\savesymbol{While}

\title[Near-optimal Optimistic Reinforcement Learning]{Near-optimal Optimistic Reinforcement Learning using Empirical Bernstein Inequalities}
\usepackage[disable]{todonotes}
\newcommand{\mytd}[1]{#1}
\renewcommand{\mytd}[1]{}

\coltauthor{%
 \Name{Aristide Tossou} \Email{yedtoss@gmail.com}\\
 \Name{Debabrota Basu} \Email{basud@chalmers.se}\\
 \Name{Christos Dimitrakakis} \Email{chrdimi@chalmers.se}\\
 \addr Chalmers %
}

\usepackage{enumerate}
\usepackage{comment}
\usepackage{pgf}
\usepackage{tikz}
\usepackage{tikzsymbols}
\usepackage{pgfplots}
\usepackage{footmisc}
\usepackage{enumitem}
\usepackage{algorithm}

\usepackage{graphicx}

\usepackage{bm}

\usepackage{adjustbox}

\usepackage{algorithmicx}
\usepackage{algpseudocode}

\usepackage{booktabs}
\usepackage{mathtools}
\usepackage{subcaption}
\usepackage{float}
\usepackage[separate-uncertainty=true, multi-part-units=single,
detect-all=true]{siunitx}
\usepackage{dsfont}
\usepackage{filecontents}
\usepackage{appendix}

\usepackage{xargs}                      
\usepackage{tabulary}

\usepackage{cleveref}
\crefformat{section}{\S#2#1#3}
\crefformat{subsection}{\S#2#1#3}
\crefformat{subsubsection}{\S#2#1#3}
\crefrangeformat{section}{\S\S#3#1#4 to~#5#2#6}
\crefmultiformat{section}{\S\S#2#1#3}{ and~#2#1#3}{, #2#1#3}{ and~#2#1#3}

\usepackage{mmacells}
\mmaDefineMathReplacement[≤]{<=}{\leq}
\mmaDefineMathReplacement[≥]{>=}{\geq}
\mmaDefineMathReplacement[≠]{!=}{\neq}
\mmaDefineMathReplacement[→]{->}{\to}[2]
\mmaDefineMathReplacement[⧴]{:>}{:\hspace{-.2em}\to}[2]
\mmaDefineMathReplacement{∉}{\notin}
\mmaDefineMathReplacement{∞}{\infty}
\mmaDefineMathReplacement{𝕕}{\mathbbm{d}}
\mmaSet{
  morefv={gobble=2},
  linklocaluri=mma/symbol/definition:#1,
  morecellgraphics={yoffset=1.9ex}
}

\DeclareMathOperator*{\argmax}{argmax}
\DeclareMathOperator*{\argmin}{argmin}

\newcommand{\BetaDis}{\textsf{Beta}}

\DeclareMathOperator{\E}{\mathbb{E}}
\DeclareMathOperator{\Prob}{\mathbb{P}}
\DeclareMathOperator{\EX}{\mathbb{E}}
\DeclareMathOperator{\Var}{Var}
\DeclarePairedDelimiter\abs{\lvert}{\rvert}%

\DeclarePairedDelimiter\size{\lvert}{\rvert}%
\DeclarePairedDelimiter\setof\{\}
\DeclarePairedDelimiter\curly\{\}

\DeclarePairedDelimiter\braket[]
\DeclarePairedDelimiter\paren()%

\DeclarePairedDelimiterX{\kldivx}[2]{(}{)}{%
  #1\;\delimsize\|\;#2%
}

\makeatletter
\newcommand{\oset}[3][0ex]{%
	\mathrel{\mathop{#3}\limits^{
			\vbox to#1{\kern-2\ex@
				\hbox{$\scriptstyle#2$}\vss}}}}
\makeatother

\makeatletter
\renewcommand{\Function}[2]{%
	\csname ALG@cmd@\ALG@L @Function\endcsname{#1}{#2}%
	\def\jayden@currentfunction{#1}%
}
\newcommand{\funclabel}[1]{%
	\@bsphack
	\protected@write\@auxout{}{%
		\string\newlabel{#1}{{\jayden@currentfunction}{\thepage}}%
	}%
	\@esphack
}
\makeatother

\makeatletter
\newcommand{\algmargin}{\the\ALG@thistlm}
\makeatother
\algnewcommand{\ParState}[1]{\State%
	\parbox[t]{\dimexpr\linewidth-\algmargin}{\strut\hangindent=\algorithmicindent \hangafter=1 #1\strut}}
\algdef{SE}[DOWHILE]{Do}{DoWhile}{\algorithmicdo}[1]{\algorithmicwhile\ #1}%

\newcommand{\SQRTFactorUCRLV}{{\ensuremath{2^{10}}}}
\newcommand{\UCRL}{{UCRL2}}
\newcommand{\KLUCRL}{{KL-UCRL}}

\newcommand{\TSDE}{{TSDE}}

\newcommand{\UCRLBERNSTEIN}{{UCRL-V}}
\newcommand{\UCRLV}{{UCRL-V}}

\newcommand{\GAMEOFSKILLEASY}{{GameOfSkill-v1}}
\newcommand{\GAMEOFSKILLHARD}{{GameOfSkill-v2}}
\newcommand{\GAMEOFSKILL}{{GameOfSkill}}

\newcommand{\TRIALS}{{\ensuremath{40}}}
\newcommand{\BigO}{{\ensuremath{\mathcal{O}}}}
\newcommand{\BigOmega}{{\ensuremath{\Omega}}}
\newcommand{\TilO}{{\ensuremath{\tilde{\mathcal{O}}}}}
\newcommand{\ConfidenceDelta}{{0.05}}
\newcommand{\HORIZONEXPERIMENTS}{{\ensuremath{2^{24}}}}

\newcommand{\gain}{V}

\newcommand{\defn}{\mathrel{\triangleq}} 
\newcommand{\round}{{round}}
\newcommand{\rounds}{{rounds}}

\newcommand{\episode}{episode}
\newcommand{\episodes}{episodes}
\newcommand{\piOpt}{\tilde{\pi}}
\newcommand{\gainOpt}{\tilde{V}}

\newcommand{\rOpt}{\tilde{r}}

\newcommand{\regret}{\mathrm{Regret}}

\newcommand{\LHS}{\text{LHS}}

\newcommand{\opt}[1]{\tilde{#1}}
\newcommand{\upperb}[1]{\hat{#1}}
\newcommand{\lowerb}[1]{\check{#1}}
\newcommand{\samples}[1]{\bm{#1}}

\newcommand{\notstate}[1][s]{\oset[-0.35ex]{\neg}{#1}}

\newcommand{\Id}{\mathbb{I}}

\newcommand{\dbcomment}[1]{}
\newcommand{\aricomment}[1]{}

\newcommand{\mdp}{M}
\newcommand{\optmdp}{\tilde{M}}

\usepackage{xr}
\makeatletter
\newcommand*{\addFileDependency}[1]{
  \typeout{(#1)}
  \@addtofilelist{#1}
  \IfFileExists{#1}{}{\typeout{No file #1.}}
}
\makeatother

\makeatletter
\newtheorem*{rep@theorem}{\rep@title}
\newcommand{\newreptheorem}[2]{%
	\newenvironment{rep#1}[1]{%
		\def\rep@title{\textbf{#2} \ref{##1}}%
		\begin{rep@theorem}}%
		{\end{rep@theorem}}}
\makeatother
\newreptheorem{theorem}{Theorem}
\newreptheorem{lemma}{Lemma}
\newreptheorem{proposition}{Proposition}

\begin{document}

\maketitle
\begin{abstract}
We study model-based reinforcement learning in an unknown finite communicating Markov decision process. We propose a simple algorithm that leverages a variance based confidence interval. We show that the proposed algorithm, \UCRLV{}, achieves the optimal regret $\TilO(\sqrt{DSAT})$ up to logarithmic factors, and so our work closes a gap with the lower bound without additional assumptions on the MDP. We perform experiments in a variety of environments that validates the theoretical bounds as well as prove \UCRLV{} to be better than the state-of-the-art algorithms.
\end{abstract}
\todo[inline]{Change back to using algorithm2e instead of the trick you used now with algoithmcx}
\todo[inline]{Remove nosubfloats option and use default jmlr subfigure?}





\section{Introduction}\label{ucrlv:sec:intro}
\paragraph{Reinforcement Learning.} In \emph{reinforcement learning}~\citep{sutton1998reinforcement}, a learner interacts with an environment over a given time horizon $T$. At each time $t$, the learner observes the current state of the environment $s_t$ and needs to select an action $a_t$. This leads the learner to obtain a reward $r_t$  and to transit to a new state $s_{t+1}$. In the \emph{Markov decision process (MDP)} formulation of reinforcement learning, the reward and next state are generated based on the environment, the current state $s_t$ and current action $a_t$ but are independent of all previous states and actions. The learner does not know the true reward and transition distributions and needs to learn them while interacting with the environment. 
There are two variations of MDP problems: discounted and undiscounted MDP.
In the discounted MDP setting, the future rewards are discounted with a factor $\gamma <  1$~\citep{brafman2002r, poupart2006analytic}. The cumulative reward is computed as the discounted sum of such rewards over an infinite horizon.
In the undiscounted MDP setting, the future rewards are not discounted and the time horizon $T$ is finite. In this paper, we focus on \emph{undiscounted MDPs}.

\paragraph{Finite communicating MDP.} An \emph{undiscounted finite MDP} $M$ consists of a finite state space $\mathcal{S}$, a finite action space $\mathcal{A}$, a reward distribution $\nu$ on bounded rewards $r \in [0,1]$ for all state-action pair $(s,a)$, and a transition kernel $p$ such that $p(s'|s,a)$ dictates the probability of transiting to state $s'$ from state $s$ by taking an action $a$.
In an MDP, at state $s_t \in \mathcal{S}$ in \round{} $t$, a learner chooses an action $a_t \in \mathcal{A}$ according to a policy $\pi_t: \mathcal{S} \rightarrow \mathcal{A}$. This grants the learner a reward $r_t(s_t,a_t)$ and transits to a state $s_{t+1}$ according to the transition kernel $p$.
The \emph{diameter} $D$ of an MDP is the expected number of \rounds{} it takes to reach any state $s'$ from any other state $s$ using an appropriate policy for any pair of states $s,s'$. More precisely,

\begin{definition}[Diameter of an MDP]
\label{def:diameter}
The diameter $D$ of an MDP $M$ is defined as the minimum number of rounds needed to go from one state $s$ and reach any other state $s'$ while acting using some deterministic policy. Formally,
\[D(M) = \max_{s\ne s', s,s' \in \mathcal{S}} \min_{\pi: \mathcal{S} \rightarrow \mathcal{A}} T(s' | s, \pi) \] where $T(s' | s, \pi)$ is the expected number of rounds it takes to reach state $s'$ from $s$ using policy $\pi$.
\end{definition}
An MDP is \emph{communicating} if it has a finite diameter $D$. 
 
Given that the rewards are undiscounted, a good measure of performance is the gain, i.e. the infinite
horizon average rewards. The gain of a policy $\pi$ starting from state s is defined by:
\[
\gain(s | \pi) \defn \limsup_{T \to \infty}\frac{1}{T} \EX\left[\sum_{t=1}^{T} r(s_t, \pi(s_t)) \mid s_1 = s\right].
\]
\citet{puterman2014markov} shows that there is a policy $\pi^*$ whose gain, $\gain^*$ is greater than that of any other policy. In addition, this gain is the same for all states in a communicating MDP.
We can then characterize the performance of the agent by its regret defined as:
\[
\regret(T) \defn \sum_{t=1}^{T} \left(\gain^* - r(s_t, a_t)\right).
\]
Regret provides a performance metric to quantify the loss in gain because of the MDP being unknown to the learner. Thus, the learner has to explore the suboptimal state-actions to learn more about the MDP while also maximising the gain as much as possible. In the literature, this is called the exploration--exploitation dilemma.
Our goal in this paper, is to design reinforcement learning algorithm that minimises the regret without a prior knowledge of the original MDP $M$ i.e. $r, p, D$ are unknown. Thus, our algorithm needs to deal with the exploration--exploitation dilemma.

\paragraph{Optimistic Reinforcement Learning.} We adopt the optimistic reinforcement learning technique for algorithm design. \emph{Optimism in the face of uncertainty (OFU)} is a well-studied algorithm design technique for resolving the exploration--exploitation dilemma in multi-armed bandits~\citep{ucbv}. 
Optimism provides scope for researchers to adopt and extend the well-developed tools for multi-armed bandits to MDPs. For discounted MDPs and Bayesian MDPs, optimism-based techniques allow researchers to develop state-of-the-art algorithms~\citep{kocsis2006bandit,silver2016mastering}. 

\citeauthor{jaksch2010near} proposed an algorithm, UCRL2, for finite communicating MDPs that uses the \emph{optimism in the face of uncertainty} framework and achieves $\TilO(DS\sqrt{AT})$\footnote{In this paper, we will use $\TilO$ notation to hide extra $\log T$  factors.} regret.
The design technique of UCRL2 can be deconstructed as follows:
\begin{enumerate}
	\item Construct a set of statistically plausible MDPs around the estimated mean rewards and transitions such that the set contains the true MDP with high probability.
	\item Compute a policy (called \emph{optimistic}) whose gain is the maximum among all MDPs in the plausible set. They developed an extended value iteration algorithm for this task.
	\item Play the computed \emph{optimistic} policy for an artificial \episode{} that lasts until the number of visits to any state-action pair is doubled. This is known as the doubling trick.
\end{enumerate} 
Follow-up algorithms further developed from this optimism perspective, such as KL-UCRL~\citep{filippi2010optimism}, REGAL.C~\citep{bartlett2009regal}, UCBVI~\citep{azar2017minimax}, SCAL~\citep{fruit2018efficient}. These proposed algorithms and proof techniques improve the regret bound of optimistic reinforcement learning up to $\TilO(D\sqrt{SAT})$, however with additional assumptions on the MDP. The best known lower bound on the regret for a unknown finite communicating MDP is $\BigOmega(\sqrt{DSAT})$, as proven by~\citet{jaksch2010near}. This leaves a gap in the literature. 
A few recent works tried to bridge this gap by either proposing proof techniques~\citep{simchowitz2019non} or algorithms or both~\citep{efroni2019tight,zhang2019regret}.
These works are either limited to the setting of episodic MDP~\citep{simchowitz2019non,efroni2019tight} or need practically unavailable side infromation, such as the upper bound of the span of the bias function~\citep{zhang2019regret}. Thus, the question of designing a practical algorithm that does not assume any special setup or extra information about the problem while achieve the regret upper bound $\TilO(\sqrt{DSAT})$ still remains open.

In this paper, we design an algorithm and a proof technique that bridge this gap  by exploiting variance based confidence bounds, and modified versions of extended value iteration algorithm and the doubling trick.
Our algorithm achieves a regret upper bound of $\TilO(\sqrt{DSAT})$ with no additional assumptions on the communicating MDP.

\paragraph{Our Contributions.} Hereby, we summarise the contributions of this paper that we elaborate in the upcoming sections.\vspace*{-.5em}
\begin{itemize}
\item We propose an algorithm, \UCRLV{} (Algorithm~\ref{algo:ucrl_bernstein}), using the optimistic reinforcement learning framework (Section~\ref{ucrlv:sec:algo}). \UCRLV{} uses an empirical Bernstein bound and a new pointwise constraint on the transition kernel to construct a crisper set of plausible MDPs than the existing algorithms. (Section~\ref{sec:constraint})
\item We propose a modified extended value iteration algorithm (Algorithm~\ref{algo:extended_vi}) that converges under the new constraints while retaining the same complexity of the extended value iteration algorithm in~\cite{jaksch2010near}. (Section~\ref{sec:mevi})
\item In Theorem~\ref{thm:final}, we prove that \UCRLV{} achieves $\TilO(\sqrt{DSAT})$ regret without imposing any additional constraint on the communicating MDP. Thus, bridging a gap in the literature (Section~\ref{ucrlv:sec:theory}).
\item We prove a correlation between  the number of visits of a policy in an MDP with the values, probabilities and diameter (Lemma \ref{lemma_probs_vs_values_sum}). This result, along with the algorithm design techniques causes the improved bound  (Section~\ref{ucrlv:sec:theory}).
\item We perform experiments in a variety of environments that validates the theoretical bounds as well as proves \UCRLV{} to be better than the state-of-the-art algorithms. (Section~\ref{ucrlv:sec:experiments})
\end{itemize}\vspace*{-.5em}
We conclude by summarising the techniques involved in this paper and discussing the possible future works they can lead to (Section~\ref{ucrlv:sec:conclusion}). 
The proofs and technical details are elaborated in the Appendix.
\section{Methodology}\label{ucrlv:sec:algo}
In this section, we describe the algorithm design methodologies used in \UCRLV{}. We categorise and elaborate the principal methods in following sections.

\begin{algorithm}[ht!]
	\caption{\UCRLBERNSTEIN{}}
	\label{algo:ucrl_bernstein}
	\begin{algorithmic}[1]
		\State \textbf{Initialization: } Assign $t \gets 1$. 
		\State Assign $N_k, N_k(s,a), N_{t_k}(s,a) \gets 0$ for all $k \geq 0$ and $(s,a)$.
		\State Observe the initial state $s_1$.
		\For{Episodes $k=1, 2, \ldots$}
		
		\ParState{$t_k \gets t$} \ParState{$N_{t_{k+1}}(s, a) \gets N_{t_{k}}(s, a)\; \forall s,a$}
		\Statex		
		\State \textbf{Compute optimistic policy $\tilde{\pi}_k$:}
				\State $\tilde{\pi}_k \gets \Call{ModifiedExtendedVI}{\frac{1}{\sqrt{t_k}}}$ ~~~~~~~~~~~~~~(Algorithm~\ref{algo:extended_vi})
				\Statex
				\State \textbf{Execute policy $\tilde{\pi}_k$:}
				\While{ $\sum_{s,a} \frac{N_k(s,a)}{\max\{1, N_{t_k}(s,a)\}} < 1$ }
				\State Play action $a_t$ .
				\State Observe reward $r_t$, present state $s_{t+1}$.
				\State $N_k \gets N_k + 1$, $N_k(s_t, a_t) \gets N_k(s_t, a_t) + 1$, 
				\State $N_{t_{k+1}}(s_t, a_t) \gets N_{t_{k+1}}(s_t, a_t) + 1$,  $t \gets t + 1$.
				\EndWhile
		\EndFor
	\end{algorithmic}
\end{algorithm}
\subsection{Constructing the Set of Statistically Plausible MDPs}\label{sec:constraint}
We construct the set of statistically plausible MDPs using two important modifications compared to previous algorithms, such as \UCRL{}~\citep{jaksch2010near}. 

\textbf{The first modification} is the \emph{construction of confidence bounds on the transitions for all subsets of next states}. Specifically, we consider an MDP to be plausible if its expected rewards $\tilde{r} \in [0, 1]$ and transitions $\tilde{p}$ satisfy the following inequalities for all state-action pair $(s,a)$ and all subset of next states $\mathcal{S}_c \subseteq \mathcal{S}$:
\begin{alignat}{2}
 \tilde{r}(s,a) - \bar{r}_{t_k}(s,a) &\leq c_{\samples{r}}(s, a; \delta^k_r, t_k)\label{eq:plausible_rewards}\\
\tilde{p}(\mathcal{S}_c |s, a) - \bar{p}_{t_k}(\mathcal{S}_c|s, a) &\leq c_{\samples{p}}(\mathcal{S}_c; s, a, \delta^k_p, t_k).\label{eq:plausible_transitions}
\end{alignat}
Here, $t_k$ denotes the number of \rounds{} at the start of present episode $k$. $\tilde{p}$ is the transition kernel of a plausible MDP.  $\tilde{p}(\mathcal{S}_c |s,a) \defn \sum_{s' \in \mathcal{S}_c} \tilde{p}(s'|s,a)$ is the plausible transition probability to the subset of states $\mathcal{S}_c$ from state-action pair $(s,a)$. 

$\delta^k_r = \frac{\delta}{4 SA \ln\paren*{t_k}}$ and $\delta^k_p = \frac{\delta}{8S^2A \ln\paren*{t_k}}$, where $1-\delta$ is the desired confidence level of the set of plausible MDPs as ensured by the upper bounds $c_{\samples{r}}$ and $c_{\samples{p}}$. We define $c_{\samples{r}}$ and $c_{\samples{p}}$ using Equation~\ref{confidence_bound}.

Here, $\samples{r} = (r_1, \ldots r_{t_k})$ represents the sequence of observed rewards till time $t_k$. Similarly, $\samples{p}$ is an indicator function such that, given a subset of states $\mathcal{S}_c$, it outputs a vector
\[
\samples{p}(\mathcal{S}_c) = (p_1(\mathcal{S}_c), \ldots p_{t_k}(\mathcal{S}_c)) ~\textrm{s.t.}~
p_t(\mathcal{S}_c) = 
\begin{cases}
1, s_{t+1} \in \mathcal{S}_c\\
0, \textrm{otherwise.}
\end{cases}
\] 
$\samples{p}(\mathcal{S}_c) $ indicates for which time-steps the next state was in $\mathcal{S}_c$.

We also define the sample mean reward and transitions for each $(s,a)$ pair to be $\bar{r}_{t_k}(s,a), 
\bar{p}_{t_k}(.|s,a)$ respectively. 
In particular using $\samples{f} =(f_1, f_2, \ldots f_{t_k})$ as a placeholder for $\samples{r}$ and $\samples{p}$,  and $N_{t_k}(s,a)$ for the number of times $(s,a)$ is played up to \round{} $t_k$, the sample mean $\bar{f}_{t_k}(s,a)$ is defined by:
$\bar{f}_{t_k}(s,a) = \sum_{t\leq t_k: (s_t, a_t) = (s,a)} \frac{f_t}{N_{t_k}(s,a)}.$

 Now $\bar{r}_{t_k}(s,a), \bar{p}_{t_k}(.|s,a)$ are defined analogously to $\bar{f}_{t_k}(s,a)$.

\textbf{The second modification} is the use of \emph{variance modulated confidence bounds based on empirical Bernstein inequalities}~\citep{empirical_bernstein_bounds}. In particular, we set the confidence bounds as $c_{\samples{r}} \defn c(\samples{r};s,a,\delta^k_r, t_k)$; $c_{\samples{p}}(\mathcal{S}_c) \defn c(\samples{p}(\mathcal{S}_c); s, a, \delta^k_p, t_k)$, where $c$ is defined by
\begin{align}
\label{confidence_bound}
&\quad c\paren*{\samples{f};s,a,\delta_f, t_k}\nonumber\\ 
&\defn \sqrt{\frac{2\Var_{t_k}(\samples{f}, s,a, t_k) \ln \frac{2}{\delta_f}}{N_{t_k}(s,a)}}+ \frac{7}{3}\frac{\ln \frac{2}{\delta_f}}{{N_{t_k}(s,a)-1}}
\end{align}
with $\Var_{t_k}(\samples{f}, s, a, t_k )$ the sample variance:
\[ \Var_{t_k}(\samples{f}, s, a, t_k ) = \sum_{t\leq t_k: (s_t, a_t) = (s,a)} \frac{\paren*{f_t - \bar{f}_{t_k}(s,a)}^2}{N_{t_k}(s,a)}. \]

Unlike the Weissman $L_1$ deviation used by \UCRL{} \citep{jaksch2010near}, our transition  vectors for a given state-action pair satisfy separate bounds for any possible subset of next states. This provides a crisper set of plausible MDPs. For example, if the empirical transition to a state is $0$, our bounds lead to an error of at most $\BigO\paren*{\frac{1}{N_{t_k}(s,a)} }$ whereas \UCRL{} could add up to $\BigO\paren*{\sqrt{\frac{S}{N_{t_k}(s,a)}} }$. 
%
%
%

\subsection{Modified Extended Value Iteration}\label{sec:mevi}
The goal of Modified Extended Value Iteration (Algorithm~\ref{algo:extended_vi}) is to find an \emph{optimistic} policy, whose average value is the maximum among all plausible MDPs. In order to develop the algorithm, we follow the schematics of \citep{jaksch2010near}. We consider an extended MDP $\tilde{M}$ with the state space $\mathcal{S}$ and a continuous action space $\mathcal{A}'$ such that for each action $a \in \mathcal{A}$, each transition vector $\tilde{p}$ satisfying \eqref{eq:plausible_transitions}, each reward function $\tilde{r}$ satisfying \eqref{eq:plausible_rewards}, there exists an action in $\mathcal{A}'$ with transition $\tilde{p}$ and mean rewards $\tilde{r}$.

Now, we can define \emph{extended value iteration}~\citep{jaksch2010near} to solve this problem:
\begin{align}
u_0(s) &= 0 \nonumber \\
u_{i+1}(s) &= \max_{a \in \mathcal{A}} \left[\tilde{r}(s,a) + \max_{\tilde{p}(.) \in \mathcal{P}(s,a)} \left[\sum_{s' \in \mathcal{S}} \tilde{p}(s')u_i(s') \right]\right]\label{eq:extended_vi}
\end{align} 
where $u_i(s)$ denotes the value of state $s$ at the $i$-th iteration 
and $\mathcal{P}(s,a)$ is the set of all possible transitions in the set of plausible MDPs satisfying \cref{eq:plausible_transitions}. 
Now, we discuss how to efficiently solve the constraint optimisation problem of \cref{eq:extended_vi}.

\textit{The outer maximum.} The maximum for any $(s,a)$ is attained by setting $\tilde{r}(s,a)$ to $\bar{r}_{t_k}(s,a) + c_{\samples{r}}(s,a; \delta^k_r, t_k)$. 

\textit{The inner maximum.} Though the set of all possible transition functions  $\mathcal{P}(s,a)$ is an infinite space, computing the maximum over it is a linear optimization problem over the convex polytope $\mathcal{P}(s,a)$, which can be solved efficiently~\citep{jaksch2010near, StrehlL08}. The intuition is to  put as much transition probability as possible to the states with maximal value  at the expense of transition probabilities to states with small value. This idea is formally established for Algorithm \ref{algo:extended_vi} in  \Cref{lemma:optimism_p} (Appendix) which shows that the value returned by Algorithm \ref{algo:extended_vi} is greater than the one obtained by any other $\tilde{p} \in \mathcal{P}(s,a)$.

\textit{Constraints.} A careful observation of \cref{eq:plausible_transitions} shows that for each $(s,a)$ there are up to $O(2^{S})$ constraints on the transition. It is computationally expensive to check each one of them. Our analysis shows that we can satisfy all $2^{S}$ constraints by just considering at most $S$ constraints. This is possible because the confidence function on the transitions $c_{\samples{p}}: \mathcal{S}_c \to c(\samples{p}(\mathcal{S}_c| s, a); \ldots)$ defined using Bernstein bounds (\Cref{eq:plausible_transitions}) is a submodular function~\citep{schrijver2003combinatorial} on the subsets of states; something we formally prove in \Cref{lemma:assumption_satisfied} (Appendix). 
We also prove in \Cref{lemma:trans} (Appendix)
that  Algorithm \ref{algo:extended_vi} satisfies all $2^{S}$ constraints when the confidence function on the transitions $c_{\samples{p}}$ is submodular.
Thus, Algorithm \ref{algo:extended_vi} correctly computes the inner maximum by checking at most $S$ constraints instead of $2^S$ constraints due to the submodularity of the confidence function on the transitions $c_{\samples{p}}$.

%
%

\begin{algorithm*}[t!]
	\caption{Modified Extended Value Iteration for Solving Equation \ref{eq:extended_vi}}
	\label{algo:extended_vi}
	
	\begin{algorithmic}[1]
	\Function{ModifiedExtendedVI}{$\epsilon$}\label{ucrlv:modifiedvi:fun}
	    \State $i \gets -1; \quad u_0(s) \gets 0\; \forall s$
	    
	   	\Do
	    \State $i \gets i + 1$
	    \State $\tilde{p}(. |s,a) \gets \Call{OptimisticTransition}{u_i, s, a}  \forall (s,a)$
	    \State $\upperb{r}(s,a) \gets \min\left\{1, \bar{r}_{t_k}(s,a) + c_{\samples{r}}(s, a; \delta^k_r, t_k)\right\}  \forall (s,a)$
	    \State $u_{i+1}(s) \gets \max_{a \in \mathcal{A}} \left\{ \upperb{r}(s,a) +  \sum_{s' \in \mathcal{S}} \tilde{p}(s'|s,a)u_i(s') \right\}  \forall s$ 
	    \State $\tilde{\pi}(s) \gets   \argmax_{a \in \mathcal{A}} \left\{ \upperb{r}(s,a) +  \sum_{s' \in \mathcal{S}} \tilde{p}(s'|s,a)u_i(s') \right\}  \forall s$
	    \DoWhile{$\max_{s}\{u_{i+1}(s)-u_i(s)\}-\min_{s}\{u_{i+1}(s)-u_i(s)\} \leq \epsilon$}
	    \State \Return $\tilde{\pi}$
	    \EndFunction
	    
	   	\Statex	
		\Function{OptimisticTransition}{$u_i$, $s$, $a$}\label{ucrlvoptimisticpfun}
		\State Sort the states in descending order such that $u_i(s'_1) \geq u_i(s'_2) \ldots \geq u_i(s'_S)$
		
		\State Let $\mathcal{S}_{j_1}^{j_2} = \{s'_{j_1}, s'_{j_1 + 1} \ldots s'_{j_2}\}\; \forall j_1 \leq j_2$ and $\mathcal{S}_{j_1}^{j_2} = \{\}\; \forall j_1 > j_2$
		
		\For{$j=1$ {\bfseries to} $S$}
		\begin{align}
		\upperb{p}(\mathcal{S}_1^j | s,a)  &\gets \bar{p}_{t_k}(\mathcal{S}_1^j | s,a) +  c_{\samples{p}}(\mathcal{S}_1^j; s,a, \delta^k_p, t_k)\nonumber\\
		\opt{p}(s'_j | s,a) &\gets \min\left\{ \upperb{p}(\mathcal{S}_1^j | s,a)-\opt{p}(\mathcal{S}_1^{j-1} | s,a), \right.
				\left. 1- \opt{p}(\mathcal{S}_{1}^{j-1} | s, a) \right\}
		\end{align}
		\EndFor
		\State \Return $\opt{p}(. | s,a)$
		\EndFunction		
	\end{algorithmic}
\end{algorithm*}

\subsection{Scheduling the Adaptive Episodes}\label{sec:doubling} 
In our analysis, we found that the standard \emph{doubling trick} that is used to start a new \episode{} can cause the length of an \episode{} to be too large. Specifically, we observe that the average number of states that are doubled during an \episode{} should be a small constant independent of $S$. However, we also need to make sure that the total number of \episodes{} is small. 

We propose to start a new \episode{} as soon as $\sum_{s,a} \frac{N_k(s,a)}{\max\{1, N_{t_k}(s,a)\}} > 1$, where $N_k(s,a)$ is the number of times $(s,a)$ is played at \episode{} $k$. Intuitively, this criterion allows us to start a new episode if either at least a new state-action is explored or the number of visits to at least one of the visited state-action pairs is doubled or the number of visits to visited state-action pairs is doubled on an average. We refer to this new condition as the \emph{extended doubling trick}. It forms a crucial part into removing an additional $\sqrt{S}$ factor compared to \UCRL{}. A more specific description is available in Algorithm~\ref{algo:ucrl_bernstein}. Theorem~\ref{thm:episodes} shows that the total number of episodes $m$ due to extended doubling trick is bounded by $\BigO(SA \log T)$. In the worse case, it does not introduce more episodes than the existing doubling trick~\citep{jaksch2010near}.

\section{Theoretical Analysis}\label{ucrlv:sec:theory}
Our proposed algorithm \UCRLV{} is formally described in \Cref{algo:ucrl_bernstein}. \Cref{algo:ucrl_bernstein} combines the three techniques described in Section \ref{ucrlv:sec:algo} to achieve the near-optimal $\TilO(\sqrt{DSAT})$ regret proven in Theorem \ref{thm:final}.
We provide a proof sketch to briefly describe the results and techniques used to obtain Theorem \ref{thm:final}.
\begin{theorem}[Upper Bound on the Regret of \UCRLV{}]\label{thm:final}
With probability at least $1-\delta$ for any $\delta \in ]0,1[$, any $T \geq 1$, the regret of \UCRLV{} is bounded by:
\begin{align*}
\mathcal{R}(T) &\leq
\SQRTFactorUCRLV{}\cdot \sqrt{D T SA \min\{\ln\paren*{D + 1}, S\} \ln\paren*{\frac{8T}{SA}} \ln \paren*{\frac{BS}{\delta}}} \\
&\quad\quad + 64DSA\ln\paren*{\frac{8T}{SA}} \ln \paren*{\frac{B}{\delta}}
\end{align*}
for $B = 32SA\ln\paren*{T}$. \dbcomment{Aristide: Check the constants once?}
\end{theorem}
\paragraph{Proof Sketch.}
	Lemma~\ref{lemma:regret_decompose} is the starting point of our proof and decomposes the regret into three terms.
	\sloppy 	
	\begin{lemma}[Regret decomposition]\label{lemma:regret_decompose}
		If the true model $M$ is within our plausible set $\mathcal{M}_k$ for each episode $k$, then with probability at least $1-\delta'$, 
		\begin{align}
		\label{ucrlv:lemma:regret_decomposition_main}
		\begin{split}
		\regret(T) &\leq \paren*{f_1(\delta^m_r)m + \sqrt{f_2(\delta')T} +  \sqrt{f_3(\delta^m_r)SAT}}\\ &+ \sum_{k=1}^{m} \sum_{s}\Delta_{k}^{\tilde{p}}(s) + \sum_{k=1}^{m} \sum_{s}\Delta_{k}^{p}(s)
		\end{split}
		\end{align}
		where $\tilde{\Delta}_{k}^{\tilde{p}}(s) = N_k(s)\cdot\sum_{s'} \paren*{\tilde{p}(s'|s)-p(s'|s)} (u_i(s') - u_i(s))$ and $\tilde{\Delta}_{k}^{p}(s) = N_k(s) \cdot \paren*{\sum_{s'}p(s'|s)u_i(s') -u_i(s)}$, $\tilde{p}$ and $p$ are respectively the transition kernels of the optimistic and true (but unknown) MDP . $N_k(s) = N_k(s, \tilde{\pi}_k(s))$ is the number of times the optimistic policy $\tilde{\pi}_k$ visits state $s$ at \episode{} $k$. $f_1, f_2, f_3$ are logarithmic functions (See Appendix for full definition).
	\end{lemma}

\paragraph{On Term 1 of \Cref{ucrlv:lemma:regret_decomposition_main}:} We obtained the first term using results and techniques proposed in~\cite{jaksch2010near}, convergence of Algorithm~\ref{algo:extended_vi}, and the upper bound on the number of episodes. 

Theorem~\ref{thm:evi_convergence} shows that the modified extended value iteration (Algorithm~\ref{algo:extended_vi}) converges to the optimal policy based on the point-wise constraints on the transitions (Equation~\ref{eq:plausible_transitions}). 
\begin{theorem}[Convergence of Modified Extended Value Iteration]\label{thm:evi_convergence}
	Let $\mathcal{M}$ be the set of all MDPs with state space $\mathcal{S}$, action space $\mathcal{A}$, transitions probabilities $\tilde{p}(s,a)$ and mean rewards $\tilde{r}(s,a)$ that satisfy \eqref{eq:plausible_rewards} and \eqref{eq:plausible_transitions} for given probabilities distribution $\bar{p}(s,a)$, $\bar{r}(s,a)$ in $[0,1]$. If $\mathcal{M}$ contains at least one communicating MDP, modified extended value iteration in Algorithm \ref{algo:extended_vi} converges. Further, stopping Algorithm \ref{algo:extended_vi} when:
	\[\max_{s}\{u_{i+1}(s)-u_i(s)\}-\min_{s}\{u_{i+1}(s)-u_i(s)\} \leq \epsilon ,\]
	the greedy policy $\pi$ with respect to $u_i$ is $\epsilon$-optimal meaning $V(\pi) \geq V^*_{\mathcal{M}} - \epsilon$
	and $\abs{u_{i+1}(s)-u_i(s)- V(\pi)} \leq \epsilon$.
\end{theorem}
Theorem~\ref{thm:episodes} states that the number of episodes incurred by our extended doubling trick is upper bounded by $O(SA\log T)$.
\begin{theorem}[Bounding the number of episodes]\label{thm:episodes}
	The number of episodes $m$ is upper bounded by
	\[m \leq SA \log_2\paren*{\frac{8T}{SA}}\]
\end{theorem}
Proof of Theorem~\ref{thm:episodes} relies on the observation that after $SA$ \episodes{} the expected number of times any state has been doubled is $SA$.

\paragraph{Bounding term 2 and 3 of \Cref{ucrlv:lemma:regret_decomposition_main}:} The bound for the second and third terms (respectively  \Cref{lemma:delta_tilde} and \Cref{lemma:delta_p} in Appendix) requires further novel results to prove. Bounding the second term requires: convergence of Algorithm~\ref{algo:extended_vi}, Theorem~\ref{thm:episodes}, definition of communicating MDP and Lemma~\ref{lemma_bound_on_visits}.
Bounding the third term requires Bernstein based martingales concentration inequalities~\citep{cesa2008improved}, Theorem~\ref{thm:episodes}, definition of communicating MDP and Lemma~\ref{lemma_bound_on_visits}.

Lemma~\ref{lemma_bound_on_visits} is a key technical result which together with Lemma~\ref{lemma_probs_vs_values_sum} allows us to remove extra $\sqrt{D}$ and $\sqrt{S}$ terms compared to \UCRL{}.
 
\subparagraph{\emph{Removing $\bm{\sqrt{D}}$.}} We prove Lemma~\ref{lemma_probs_vs_values_sum} to remove a factor of $\sqrt{D}$ compared to \UCRL{}. 
Lemma \ref{lemma_probs_vs_values_sum} bounds the correlation of two quantities: the expected number of transitions from a set of states to another by playing an optimistic policy for $y$ steps and the difference of values of these two states. If $y=D$, A trivial bound for the left side would be $D^2$.

 Our analysis naturally have the transitions due to our variance based error and also the difference of values. So, to be able to use Lemma ~\ref{lemma_probs_vs_values_sum}, it just remains to relate the number of visits $N_k(s)$ to a state $s$ in $k$-th episode, to the number of visits $\tilde{u}^c_y$ to the state $s$ by playing the optimistic policy in the optimistic MDP. This is achieved by Lemma  \ref{lemma_bound_on_visits}.
\begin{lemma}\label{lemma_probs_vs_values_sum}
Let $\mathcal{S}_0$ and $\mathcal{S}_1$ any two non empty subset of states. Let $s \in \mathcal{S}_0$. We have:	
	\[ \tilde{u}^c_y(\mathcal{S}_0 | s) \cdot \abs*{ \min_{s' \in \mathcal{S}_1}u_i(s')-u_i(s)} \min_{s' \in \mathcal{S}_0} \tilde{p}(\mathcal{S}_1 | s') \leq y   \]
	
	where $\tilde{u}^c_y(\mathcal{S}_0 | s)$ represents the total expected number of time the optimistic policy $\tilde{\pi}_k$ visits the states $s' \in \mathcal{S}_0$ when starting from state $s$ and playing for $y$ steps in the optimistic MDP $\tilde{M}_k$. $y = \min\{x, D\}$ with $x$ being the number of \rounds{} you need to play, when starting from $s$, to visit any state in $\mathcal{S}_0$ for $\frac{1}{\min_{s' \in \mathcal{S}_0} \tilde{p}(\mathcal{S}_1 | s')}$ times in expectation.
\end{lemma}	

\begin{lemma}
\label{lemma_bound_on_visits}
Let $\mathcal{S}_0$ any subset of states, $k$ any episode in which the true MDP is inside the plausible set $\mathcal{M}_k$ such that $N_k \geq 32 \cdot y \cdot \max\{C_1^{\delta_p}, C_2^{\delta_p}\}$ for any $y$ . We have with probability at least $1-S\delta_p$:

\[N_k(\mathcal{S}_0) \leq  12 \frac{N_k}{y} \max_{s \in \mathcal{S}}\tilde{u}^c_y(\mathcal{S}_0|s)\]
with $N_k$ the length of episode $k$ and $C_1^{\delta_p}, C_2^{\delta_p}$ some constants dependent on our confidence interval.
\end{lemma}

\subparagraph{\emph{Removing $\bm{\sqrt{S}}$.}} Using Lemma~\ref{lemma_probs_vs_values_sum} and \ref{lemma_bound_on_visits} as explained above would lead to an additional $\sqrt{S}$ factor. 

In order to avoid it, we partition the state space $\mathcal{S}$ with respect to a fixed state $s$ based on the values $u_i(s' \in \mathcal{S}) - u_i(s)$ into intervals $\mathcal{I}^u_{+} = \setof{]\frac{D}{2}, D], ]\frac{D}{4}, \frac{D}{2}] , ]\frac{D}{8}, \frac{D}{4}], \ldots }$ and $\mathcal{I}^u_{-} = \setof{[-D, -\frac{D}{2}[, [-\frac{D}{2}, -\frac{D}{4}[ , [-\frac{D}{4}, -\frac{D}{8}[, \ldots }$. These intervals are constructed in a way that the ratio between upper and lower endpoint is 2. These intervals and the corresponding partition of states lead to a geometric sum bounded by $\log_2 D$ rather than $O(S)$ (See Derivations from \Cref{eq:actual_use_n-k-s} in Appendix).


These results together provide us the desired $\tilde{O}(\sqrt{DSAT})$ bound on regret in Theorem~\ref{thm:final}.

\section{Experimental Analysis} \label{ucrlv:sec:experiments}
\begin{figure}[t!]
	\centering
	\begin{subfigure}{0.5\textwidth}
		\includegraphics[width=\textwidth]{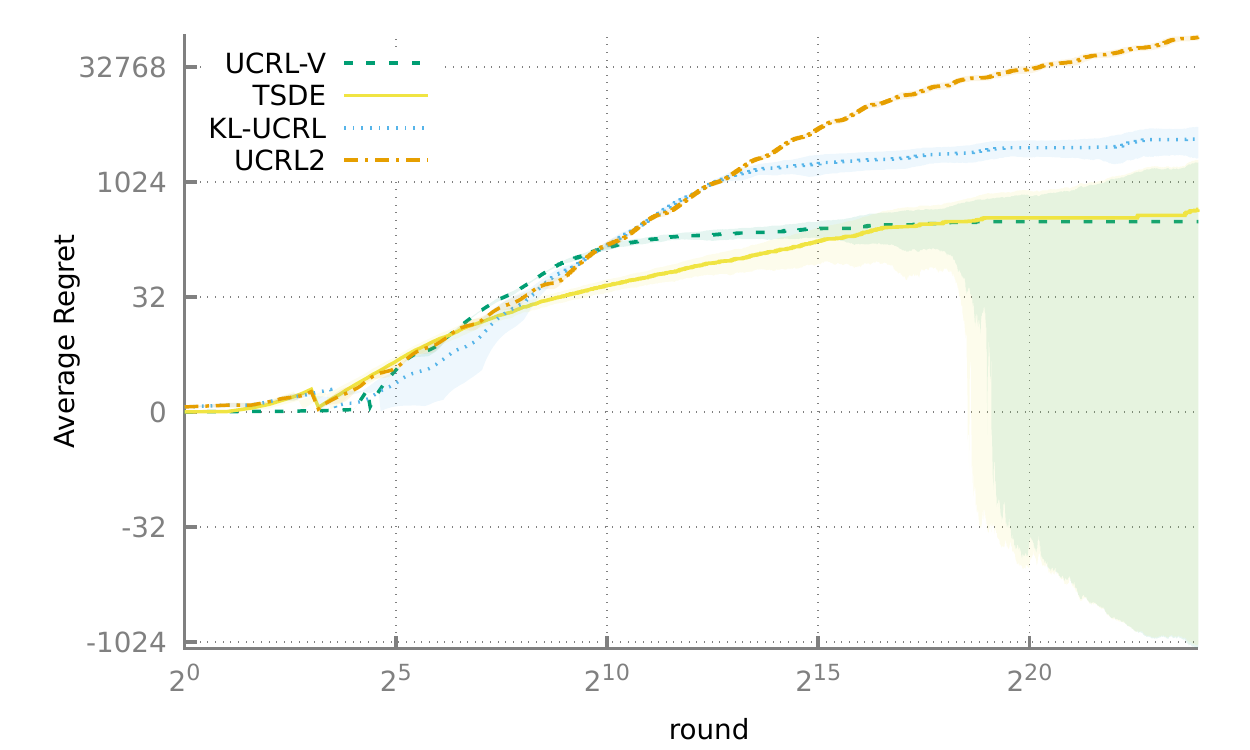}
		\caption{RiverSwim}
		\label{ucrlv:fig:RiverSwim}
	\end{subfigure}%
	\begin{subfigure}{0.5\textwidth}
			\includegraphics[width=\textwidth]{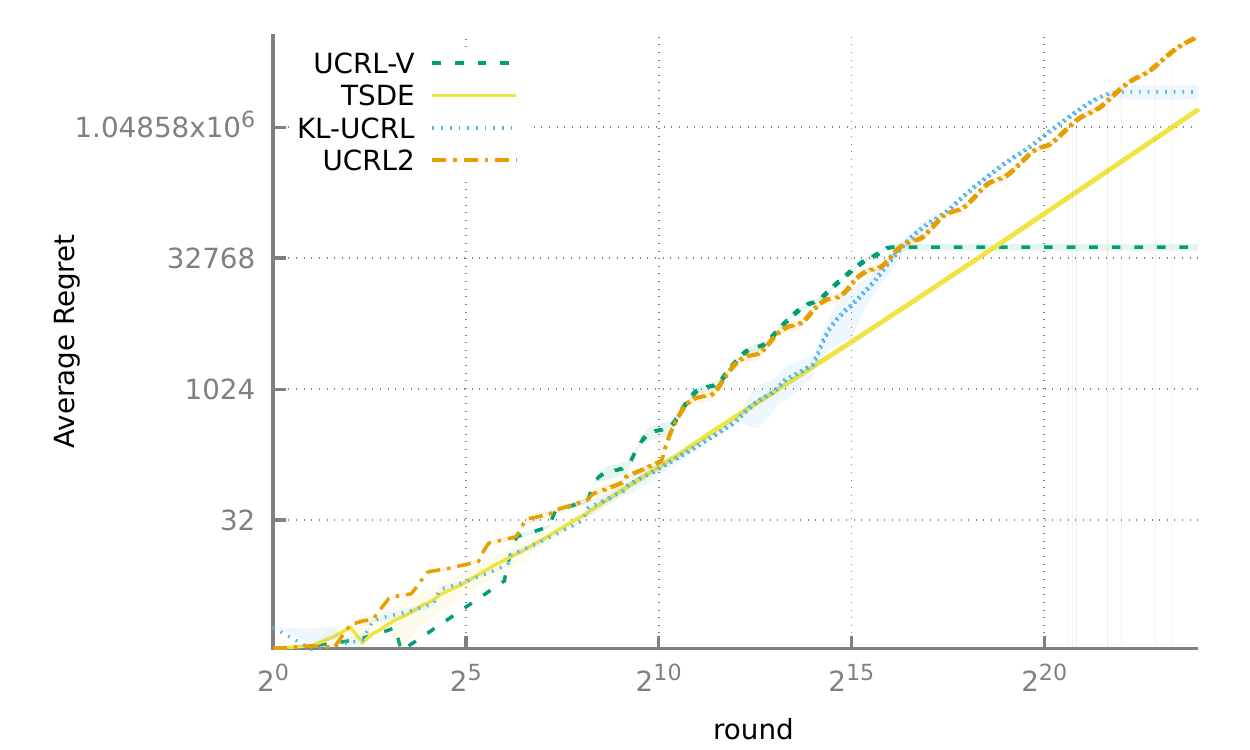}
			\caption{GameOfSkill-v1}
			\label{ucrlv:fig:GameOfSkill-t}
		\end{subfigure}\\%
	\begin{subfigure}{0.5\textwidth}
			\includegraphics[width=\textwidth]{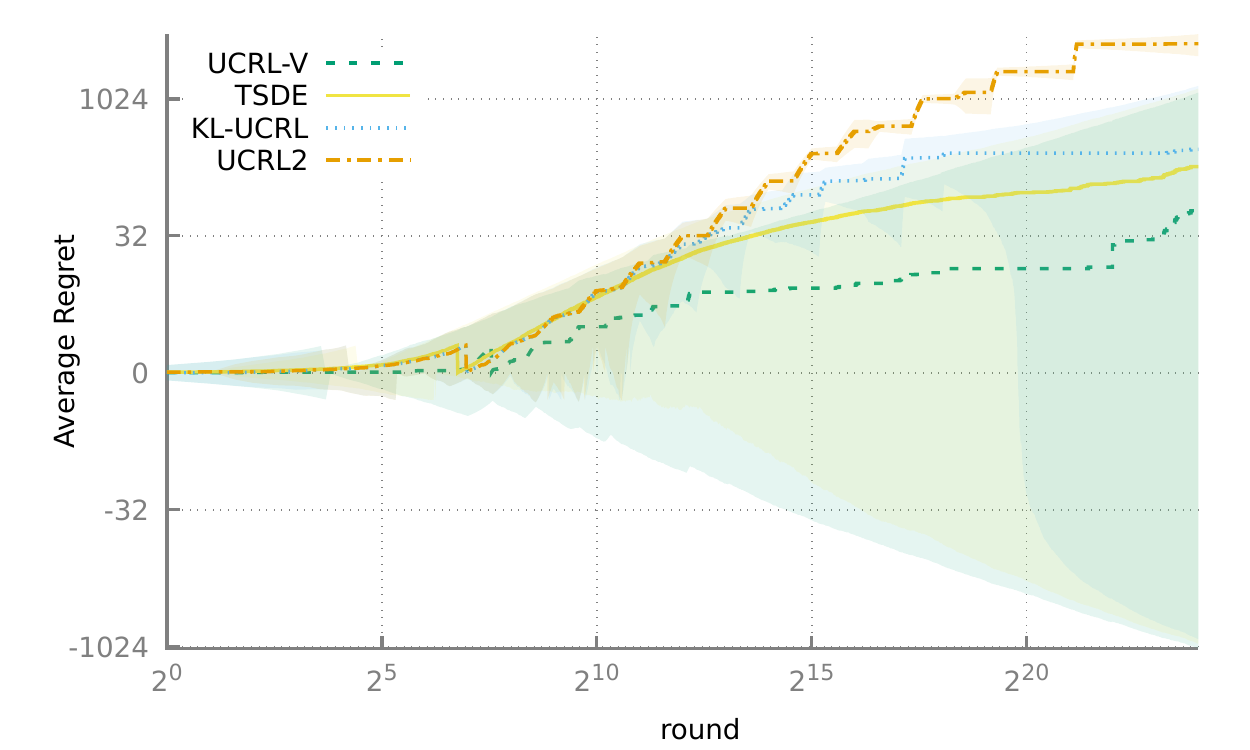}
			\caption{Bandits}
			\label{ucrlv:fig:bandit}
		\end{subfigure}%
	\begin{subfigure}{0.5\textwidth}
		\includegraphics[width=\textwidth]{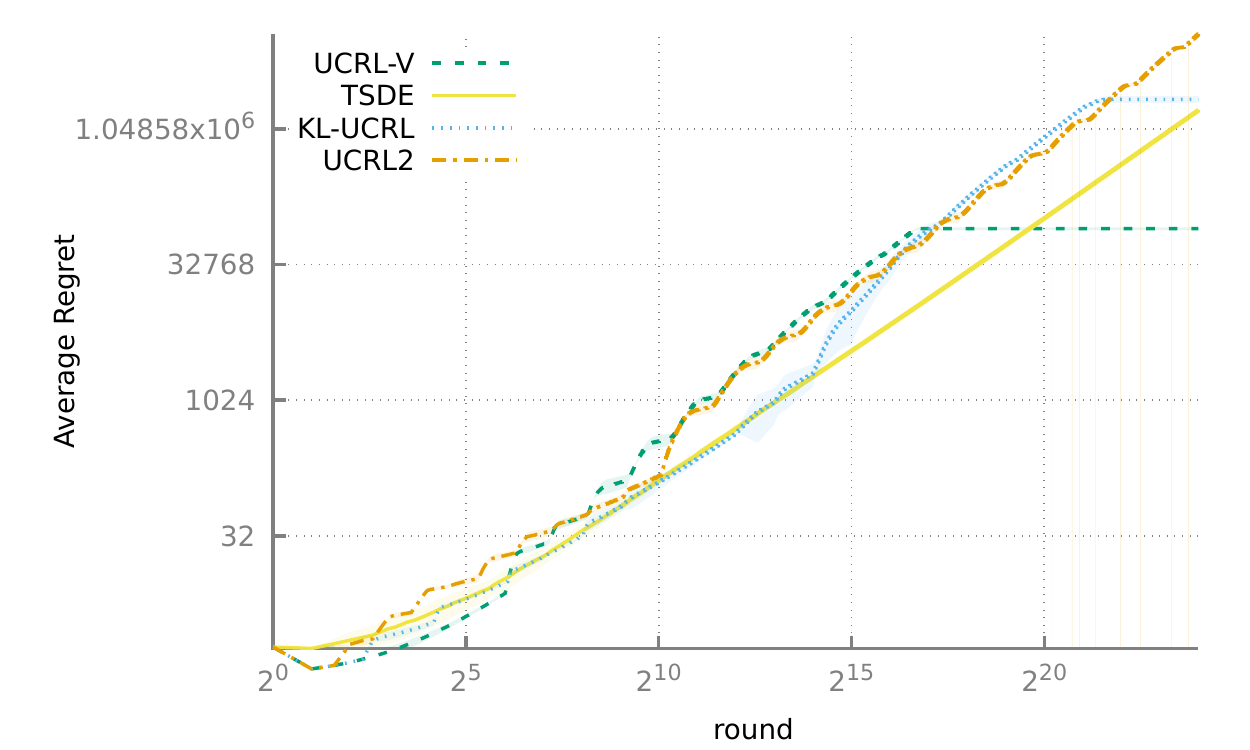}
		\caption{GameOfSkill-v2}
		\label{ucrlv:fig:GameOfSkill-f}
	\end{subfigure}
	\caption{Time evolution of average regret for \UCRLV{}, \TSDE{}, \KLUCRL{}, and \UCRL{}.} 
	\label{ucrlv:fig:regrets} 
\end{figure}
We empirically evaluate the performance of \UCRLV{} in comparison with that of \KLUCRL{} \citep{filippi2010optimism}, \UCRL{} \citep{jaksch2010near}, and \TSDE{} \citep{ouyang2017learning} that is a variant of posterior sampling for reinforcement learning suited for infinite horizon problems. Section \ref{sec:environments} describes the environments used for the experiments. 
\Cref{ucrlv:fig:regrets} illustrates the evolution of the average regret along with standard deviation. Figure~\ref{ucrlv:fig:regrets} is a log-log plot where the ticks represent the actual values.

\textbf{Experimental Setup.} The confidence hyper-parameter $\delta$ of \UCRLV{}, \KLUCRL{}, and \UCRL{} is set to $\ConfidenceDelta{}$. 
\TSDE{} is initialized with independent $\BetaDis(\frac{1}{2}, \frac{1}{2})$ priors for each reward $r(s,a)$ and a Dirichlet 
prior with parameters $(\alpha_1, \ldots \alpha_S)$ for the transition functions $p(.|s,a)$, where $\alpha_i = \frac{1}{S}$. We plot the average regret of each algorithm over $T = \HORIZONEXPERIMENTS{}$ \rounds{} computed using \TRIALS{} independent trials. 

\textbf{Experimental Protocol.}
While comparing different algorithms, we take two measures to eliminate unintentional bias and variance introduced by experimental setup. Firstly, the true ID of each state and action is masked by randomly shuffling the sequence of states and actions. This is done independently for each trial so as to make sure that no algorithm can coincidentally benefit from the numbering of states and actions.
Secondly, similar to other authors~\citep{mcgovern1998macro}, we eliminate unintentional variance in our results by using the same pseudo-random seeds when generating transitions and rewards for each trial. Specifically, for each trial, every state-action pair's pseudo-random number generator is initialised with the same initial seed. This setup ensures that if two algorithms take the same actions in the same trial, they will generate the same transitions and thus, reduces variance.

\textbf{Implementation Notes on \UCRLV{}.}
We maintained the empirical means and variance of the rewards and transitions efficiently using \citeauthor{welford1962note}'s online algorithm. Also, the empirical mean transition $\bar{p}$ to any subset of next state is the addition of its constituent and the corresponding variance is $\bar{p} \cdot (1-\bar{p})$. As a result, bookkeeping $SA$ values is enough for our algorithm. Additionally, the time complexity for $T$ runs of UCRL-V is $O(TA+(S^3A^2\ln T) \cdot N_{\textrm{MEVI}})$, where $N_{\texttt{MEVI}}$ is the number of operations required for convergence of Algorithm~\ref{algo:extended_vi} (ref. Section 3.1.5 in~\citep{strehl2008analysis}; Section 4.1 in~\citep{efroni2019tight}). This matches the time complexity of \UCRL{} in the worst-case.

\subsection{Description of Environments}
\label{sec:environments}
\textbf{RiverSwim.}
RiverSwim consists of six states arranged in a chain (ref. Figure 1 in \citet{osband2013more}). 
The agent begins
at the far left state and at every \round{}, has the choice to swim left or right. Swimming left
(with the current) is always successful, but swimming right (against the current) often fails.
The agent receives a small reward for reaching the leftmost state, but the optimal policy is
to attempt to swim right and receive a much larger reward.
The transitions are the same as in  \citep{osband2013more}. To make the problem a little tougher, we increased the rewards of the leftmost state to $0.208$ and the reward of the rightmost state is set at $0.5$. This decreases the difference in the value of the optimal and sub-optimal policies so as to make it harder for an agent to distinguish between them. Figure~\ref{ucrlv:fig:RiverSwim} shows the results.

\textbf{Bandits.}
This is a standard stochastic bandit problem with two arms. One arm draws rewards from a Beta distribution $\BetaDis(0.8 + \sqrt[4]{\frac{1}{T}}, 0.2 - \sqrt[4]{\frac{1}{T}} )$ while the other always gives $0.8$. Figure \ref{ucrlv:fig:bandit} show the results in this environment.

%
%


\textbf{\GAMEOFSKILLEASY{}.}
This environment is inspired by real-world scenarios in which a) one needs to take a succession of decisions before receiving any explicit feedback b) taking a wrong decision can undo part of the right decisions taken so far.

This environment consists of 20 states in a chain with two actions available at each state (\emph{left} and \emph{right}). Taking the \emph{left} action always transits to the correct state. However, when going to the right from a state $s$ it only succeeds with probability $\frac{1}{25}$ and with probability $1-\frac{1}{25}$, one stays in $s$. The rewards at the leftmost state for the action \emph{left} is $0.8$ whereas the reward at the rightmost state for the action \emph{right} is $0.9$. All other rewards are $0$.

\textbf{\GAMEOFSKILLHARD{}.}
This is essentially the same as \GAMEOFSKILLEASY{} with the difference that going left now send you back to the leftmost state and not just the previous state.



\subsection{Results and Discussion}
Figure~\ref{ucrlv:fig:bandit} shows an important result since to solve a larger MDP one faces at least $SA$ bandits problems. Figure~\ref{ucrlv:fig:bandit} illustrates the main reason why \UCRLV{} enjoys a better regret. It is able to efficiently exploit the non-hardness of the bandit problem tested. In contrast, \UCRL{} does not exploit the structure of the problem at hand and instead obtain a problem independent performance. Both \KLUCRL{} and \TSDE{} are also able to exploit the problem structure but are out-beaten by \UCRLV{}.

The results on the 6-states RiverSwim MDP in Figure~\ref{ucrlv:fig:RiverSwim} illustrates the same story as in the bandit problem for \UCRLV{} compared to \UCRL{} and \KLUCRL{}. However, \TSDE{} outperforms \UCRLV{} and much of gain comes from the first $2^{10}$ \rounds{}. It seems that \TSDE{} quickly moves to the seemingly good region of the state space without properly checking the apparent bad region. This can lead to catastrophic behavior as illustrated by the results on the more challenging \GAMEOFSKILL{} environments.

In both \GAMEOFSKILL{} environments (Figure~\ref{ucrlv:fig:GameOfSkill-f} and ~\ref{ucrlv:fig:GameOfSkill-t}), \UCRLV{} significantly outperforms all other algorithms. Indeed, \UCRLV{} spends the first few \rounds{} trying to learn the games and is able to do so in a reasonable time. Comparatively, \TSDE{} never tries to learn the game. Instead, \TSDE{} quickly decides to play the region of the state space that is apparently the best. However, this region turns out to be the worst region and TSDE{} never recovers. Both \KLUCRL{} and \UCRL{} attempts at learning the game. \UCRL{} didn't complete its learning before the end of the game. While \KLUCRL{} takes a much longer time to learn.
\begin{figure}
	\centering
	\includegraphics[scale=0.5]{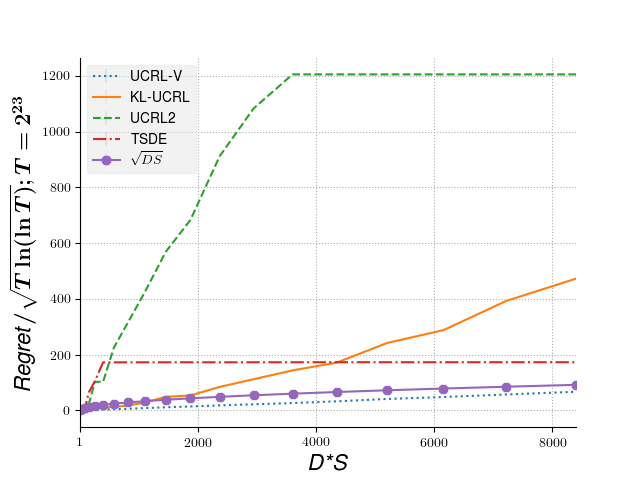}
	\caption{Growth of average regret for \UCRLV{}, \TSDE{}, \KLUCRL{}, and \UCRL{} with respect to $DS$.}\label{ucrlv:fig:DS}
\end{figure}
\subsection{Validating the Regret Bound in terms of $D$ and $S$}
In order to empirically validate the regret upper bound, we run \UCRLV{}, \TSDE{}, \KLUCRL{}, and \UCRL{} on GameOfSkill-v2 for different values of $D \times S$ and a horizon $T = 2^{23}$. 
We tune the parameters of GameOfSkill-v2 such that if $DS = x$, then $S \approx x^{1/3}; D \approx x^{2/3}$. 
We run 50 trials for each pair of $D$ and $S$.
Figure~\ref{ucrlv:fig:DS} instantiates the corresponding experimental results. We plot the average cumulative regret incurred by the algorithms divided by $\sqrt{T\ln T}$ on the $y$-axis and different value of $DS$ on the $x$-axis. 

Figure~\ref{ucrlv:fig:DS} shows that the scaling of UCRL-V with respect to $DS$ is better than the competing algorithms. The line plotting $\sqrt{DS}$ in Figure~\ref{ucrlv:fig:DS} validates that the theoretical upper bound on expected regret for \UCRLV{} is $O(\sqrt{DS})$ whereas the competing algorithms scale worse than $O(\sqrt{DS})$.

\section{Conclusion}
\label{ucrlv:sec:conclusion}
Leveraging the empirical variance of rewards and transition functions to compute the upper confidence bound provides more control over the optimism used in \UCRLV{} algorithm. This trick provides us a narrower set of statistically plausible set of MDPs. Along with the modified extended value iteration and an \emph{extended doubling trick} using the idea of average number of states doubled, provides \UCRLV{} a near-optimal regret guarantee based on the empirical Bernstein inequalities~\citep{empirical_bernstein_bounds}. As \UCRLV{} achieves the $\TilO(\sqrt{DSAT})$ bound on worst case regret, it closes a gap in the literature following the lower bound proof of~\citep{jaksch2010near}.
Experimental analysis over four different environments illustrates that \UCRLV{} is strictly better than the state-of-the-art algorithms.

Due to the relation between KL-divergence and variance, we would like to explore if a variant of \KLUCRL{} can guarantee a near-optimal regret. Also, it will be interesting to explore the possibility of guaranteeing a near-optimal regret bound for posterior sampling.
Finally, it would be interesting to explore how one can re-use the idea of \UCRLV{} for non-tabular settings such as with linear function approximation or deep learning.



\section{Notations}
\begin{table}[H]
\centering
		\renewcommand{\arraystretch}{1.3}
		\begin{tabular}{r c p{10cm} }
			\toprule
			$\mathcal{S}$ & $\defn$ & State space\\
			$\mathcal{A}$ & $\defn$ & Action space\\
			$\mathcal{S}_c$ & $\defn$ & A subset of states i.e. $\mathcal{S}_c \subseteq \mathcal{S}$\\
			$T$ & $\defn$ & Length of time horizon $T$\\
			$\regret(T)$ & $\defn$ & Regret for a given horizon $T$\\
			$\mdp$ & $\defn$ & Original MDP\\
			$\optmdp$ & $\defn$ & Optimistic MDP\\
			$D$ & $\defn$ & Diameter of original MDP $\mdp$\\
			$m$ & $\defn$ & Total number of episodes\\
			$\delta_p^m, \delta_p$ & $\defn$ & Final confidence values for the transitions at episode $m$\\
			$\delta_r^m, \delta_r$ & $\defn$ & Final confidence values for the rewards at episode $m$\\ 
			${N}_{k}$ & $\defn$ & Length of episode $k$\\ 
			$p(.|s,a)$ & $\defn$ & Transition kernel of original MDP given state $s$ and action $a$\\
			$\tilde{p}(.|s,a)$ & $\defn$ & Optimistic transition kernel given state $s$ and action $a$\\
			$\bar{p}(.|s,a)$ & $\defn$ & Empirical average of transition kernels given state $s$, action $a$\\  
			$r(s,a)$ & $\defn$ & Reward in original MDP given state $s$ and action $a$\\
			$\tilde{r}(s,a)$ & $\defn$ & Optimistic reward given state $s$ and action $a$\\
			$\bar{r}(s,a)$ & $\defn$ & Empirical average of rewards given state $s$, action $a$\\
			$N_{t_k}(s,a)$ & $\defn$ & Number of times $(s,a)$ is played up to \round{} $t_k$\\
			$N_k(s)$  & $\defn$ & Number of times a state $s$ is visited in episode $k$\\
			$u^k_i(s), u_i(s)$  & $\defn$ & Final value function obtained after $i$ iterations of Algorithm~\ref{algo:extended_vi} in episode $k$ for the set of all plausible MDPs. \\
			$\tilde{\Delta}_{k}^{\tilde{p}}(s)$& $\defn$ &$N_k(s) \cdot\sum_{s'} \paren*{\tilde{p}(s'|s)-p(s'|s)} \biggl(u_i(s') - u_i(s)\biggr)$\\ $\tilde{\Delta}_{k}^{p}(s)$& $\defn$& $N_k(s) \cdot \paren*{\sum_{s'}p(s'|s)u_i(s') -u_i(s)}$\\
			$\mathcal{I}^p$  & $\defn$ &$\{ ]\frac{1}{2}, 1]; ]\frac{1}{4}, \frac{1}{2}]; \ldots ]\frac{1}{D}, \frac{2}{D}], \ldots \}$\\
			$\mathcal{I}^u_{+} $  & $\defn$ &$\setof{]\frac{D}{2}, D], ]\frac{D}{4}, \frac{D}{2}] , ]\frac{D}{8}, \frac{D}{4}], \ldots }$\\
			$\mathcal{I}^u_{-} $  & $\defn$ &$\setof{[-D, -\frac{D}{2}[, [-\frac{D}{2}, -\frac{D}{4}[ , [-\frac{D}{4}, -\frac{D}{8}[, \ldots }$\\
			$\tilde{u}^c_y(\mathcal{S}_0 | s)$, $\tilde{u}^c_y(\mathcal{S}_0 ; s)$ & $\defn$ & Expected number of time the optimistic policy $\tilde{\pi}_k$ visits any states $s' \in \mathcal{S}_0$ when starting from state $s$ and playing for $y$ steps in the optimistic MDP $\tilde{M}_k$\\
			\bottomrule
		\end{tabular}
	\caption{Table of Notations}
	\label{tab:TableOfNotationForMyResearch}
\end{table}

\section{Proofs of Section~\ref{ucrlv:sec:theory} (Theoretical Analysis)}
\label{ucrlv:sec:proof}

\newcounter{ucrlvpara}
\newcommand\ucrlvproofstep[1]{\par\refstepcounter{ucrlvpara}\textbf{Step \theucrlvpara:\space#1}}
All the proof sketches assume bounded rewards $r \in [0,1]$.

\subsection{Proof of \UCRLV{}}
The proof of \UCRLV{} relies on a generic proof provided in Section \ref{sec:generic_proof} for any algorithm that uses the same structure as \UCRLV{} with a plausible set containing the true model with high probability whose error function is submodular and bounded in specific a form.

As a result, in this section we simply show that \UCRLV{}  satisfies the requirements in the generic proof of Section \ref{sec:generic_proof}. For that, we simply have to show that our plausible set contains the true rewards and transition for each $(s,a)$ with high probability then  express the maximum errors in a specific form and show the submodularity. We start with the rewards then move on to the transitions.

For the rewards, using  \Cref{theo:empirical_bernstein} and replacing the sample variance by $\frac{1}{4}$, we have for $N_{t_k}(s,a) \geq 2$ with probability at least $1-\delta^k_r$:

\begin{align}
	\upperb{r}(s,a) - \bar{r}(s,a) &= \sqrt{\frac{2\Var_{t_k}{(\samples{r}(s,a))} \ln \frac{2}{\delta^k_r}}{N_{t_k}(s,a)}}+ \frac{7}{3}\frac{\ln \frac{2}{\delta^k_r}}{{N_{t_k}(s,a)-1}}\notag\\
	&\leq \sqrt{\frac{\ln 2/\delta^k_r}{2N_{t_k}(s,a)}} + \frac{14\ln 2/\delta^k_r}{3N_{t_k}(s,a)} \label{eq:reward_bound}
\end{align}
We obtain the last inequality since $N_{t_k}(s,a) - 1 \geq 0.5 N_{t_k}(s,a)$ for $N_{t_k}(s,a) \geq 2$.
Similarly, using Theorem \ref{theo:empirical_bernstein} for the transitions of each state-action and replacing the sample variance by the true variance using Theorem \ref{theo:sample_variance_bound} and the union bound in Fact \ref{union_bound}, we have with probability at least $1- 2\delta^k_p$ (individually for each $(s,a)$ and subset of next states $\mathcal{S}_c \subseteq \mathcal{S}$):
\begin{align}
\upperb{p}(\mathcal{S}_c| s,a) - \bar{p}(\mathcal{S}_c| s,a) &= \sqrt{\frac{2\Var_{t_k}{(\samples{p}(\mathcal{S}_c | s,a))} \ln \frac{2}{\delta^k_p}}{N_{t_k}(s,a)}}+ \frac{7}{3}\frac{\ln \frac{2}{\delta^k_p}}{{N_{t_k}(s,a)-1}}\\
&= \sqrt{\frac{2\bar{p}(\mathcal{S}_c| s,a))(1-\bar{p}(\mathcal{S}_c | s,a))\ln \frac{2}{\delta^k_p}}{N_{t_k}(s,a)}}+ \frac{7}{3}\frac{\ln \frac{2}{\delta^k_p}}{{N_{t_k}(s,a)-1}}\\
&\leq \sqrt{\frac{2 (p(\mathcal{S}_c | s,a)(1-p(\mathcal{S}_c | s,a))\ln 2/\delta^k_p}{N_{t_k}(s,a)}} + \frac{14\ln 2/\delta^k_p}{3N_{t_k}(s,a)} \\
&\quad \quad +  \sqrt{\frac{4\ln 2/\delta^k_p\ln1/\delta^k_p}{N_{t_k}(s,a)(N_{t_k}(s,a)-1)} }\\
&\leq \sqrt{\frac{2 (p(\mathcal{S}_c | s,a)(1-p(\mathcal{S}_c | s,a))\ln 2/\delta^k_p}{N_{t_k}(s,a)}} + \frac{23\ln 2/\delta^k_p}{3N_{t_k}(s,a)} 
\label{eq:transition_bound}
\end{align}
Furthermore let's observe that the bound in \eqref{eq:reward_bound} and \eqref{eq:transition_bound} works for $N_{t_k}(s,a)  \leq 1$ since the second term is greater than $1$ when $N_{t_k}(s,a) \leq 1$. This means the bound works for any $N_{t_k}(s,a) \geq 0$.

As a result, the proof in Section \ref{sec:generic_proof} applies where

\begin{align}
C^{\delta_r^k}_1 &= \frac{\ln 2/\delta_r^k}{2}\\
C^{\delta_r^k}_2 &= \frac{14\ln 2/\delta_r^k}{3}\\
C^{\delta_p^k}_1 &= 2 \ln 2/\delta_p^m\\
C^{\delta_p^k}_2 &= \frac{23 \ln 2/\delta_p^k}{3}
\end{align}

\begin{lemma}
\label{lemma:assumption_satisfied}
The upper bound $c_{\samples{p}}: \mathcal{S}_c \to c(\samples{p}(\mathcal{S}_c| s,a), \delta_p^k)$ in RHS of \eqref{eq:plausible_transitions} is a submodular set function on the set of all states.
\end{lemma}
\begin{proof}
We perform the proof for any given state-action pair $s,a$. Thus, for brevity, we omit mentioning $s,a$ and $\delta_p^k$ while designating probabilities and the bound $c$ given $s,a$. Specifically, we write $c(\samples{p}(\mathcal{S}_c| s,a), \delta_p^k) = c(\mathcal{S}_c)$ for this proof.

We observe that for any subset of states $\mathcal{S}_c$, the upper bound $c(\mathcal{S}_c) = \sqrt{C_0 \bar{p}(\mathcal{S}_c) (1-\bar{p}(\mathcal{S}_c) )} + C_1$ with $C_0$ and $C_1$ being constants independent of $\mathcal{S}_c$.
As a result $c(\mathcal{S}_c) = f(\bar{p}(\mathcal{S}_c))$ with $f: z \to \sqrt{C_0 z(1-z)} + C_1$. Note that the function $f$ is concave. Also, $\mathcal{S}_c \to \bar{p}(\mathcal{S}_c)$ is monotone since for any $X \subseteq Y$, $\bar{p}(X) \leq \bar{p}(Y)$. Furthermore, $\mathcal{S}_c \to \bar{p}(\mathcal{S}_c)$ is modular since for any $X,Y,x : X \subseteq Y, x \in X$ we have $\bar{p}(Y)-\bar{p}(Y \setminus x) = \bar{p}(x) = \bar{p}(X)-\bar{p}(X \setminus x)$. As a result, $c_{\samples{p}}$ is the composition of a concave function with a monotonic modular function. Using \Cref{ucrlv:theorem_modular_concave_composition}, we can then conclude that $c_{\samples{p}}$ is submodular.
\end{proof}

\begin{corollary}[Submodularity of $\upperb{p} = \bar{p} + c_{\samples{p}}$]
The function $\mathcal{S}_c \to \upperb{p}(\mathcal{S}_c)$ is submodular on the set of all states. 
\end{corollary}
\begin{proof}
Using \Cref{lemma:assumption_satisfied} we know that $c_{\samples{p}}$ is submodular. Furthermore, we can easily check (see proof of \Cref{lemma:assumption_satisfied}) that $\bar{p}$ is also submodular. Using \Cref{ucrlv:theorem_modular_summation} stating that the sum of two submodular function is submodular, we can conclude that $\bar{p} + c_{\samples{p}}$  is submodular.
\end{proof}

%
%
%

\subsection{Generic Proof For Regret Bound}
\label{sec:generic_proof}
In this section, we prove in a generic way, the regret for Algorithm \ref{algo:ucrl_bernstein}.  
%
In particular, the following proof relates to any method that uses Algorithm \ref{algo:ucrl_bernstein} and uses at each episode $k$, a set of plausible models specified by $\upperb{r}, \upperb{p}$ with the following properties:

\begin{enumerate}[label=\textbf{R.\arabic*},ref=R.\arabic*]
\item $r(s,a) \leq \upperb{r}(s,a) \text{ w.p. } 1-\delta^k_r \quad \text{individually for any } (s,a)$ \label{req:r_opt}

\item $p(\mathcal{S}_c | s,a) \leq \upperb{p}(\mathcal{S}_c | s,a) \text{ w.p. } 1-\delta^k_p \quad  \text{individually for any }(s,a, \mathcal{S}_c \subseteq \mathcal{S})$ \label{req:p_opt}

\item $\upperb{r}(s,a) - \bar{r}(s,a) \leq C(1/2,C^{\delta^k_r}_1,C_2^{\delta^k_r}) \quad \forall (s,a)$ \label{req:r_opt_bound}


\item $\upperb{p}(\mathcal{S}_c | s,a) - \bar{p}(\mathcal{S}_c | s,a) \leq C(p(\mathcal{S}_c | s,a),C_1^{\delta^k_p},C_2^{\delta^k_p}) \quad \forall (s,a), \mathcal{S}_c \subseteq \mathcal{S}$ \label{req:p_opt_bound}

\item The function $\mathcal{S}_c \to \upperb{p}(\mathcal{S}_c | s,a)$ is submodular $\quad\forall (s,a)$.\label{req:submodular}

\end{enumerate}
where $\text{w.p.}\quad 1-\delta^k_r$ (or $1-\delta^k_p$) means with probability at least $1-\delta^k_r$ (or $1-\delta^k_p$), $r$ and $p$ are respectively the rewards and probabilities of the true model, and $\bar{r}$ and $\bar{p}$ are the empirical mean observation of $r$ and $p$ respectively.

\[C(x(s,a),C_1,C_2) = c_1(x(s,a),C_1) + c_2(x(s,a),C_2).\]
with:
\begin{align}
c_1(x(s,a),C_1)  &= \sqrt{\frac{C_1 \cdot x(s,a) \cdot \paren*{1-x(s,a)}}{N_{t_k}(s,a)}}\\
c_2(x(s,a),C_2) &= \frac{C_2}{N_{t_k}(s,a)}.
\end{align}

\textbf{Proof Overview.}
We start similarly as in \citet{jaksch2010near} by decomposing the regret into two main parts $\tilde{\Delta}_{k}^{\tilde{p}}(s)$ and $\tilde{\Delta}_{k}^{{p}}(s)$ as shown in Lemma \ref{lemma:regret_decompose}. 

In Lemma \ref{lemma:delta_tilde}, we show how to bound the part $\tilde{\Delta}_{k}^{\tilde{p}}(s)$. One of the main idea in the proof is to assign the states $s'$ based on the values $u_i(s') - u_i(s)$  into an infinite set of bins $\mathcal{I}^u_{+} = \setof{]\frac{D}{2}, D], ]\frac{D}{4}, \frac{D}{2}] , ]\frac{D}{8}, \frac{D}{4}], \ldots }$, $\mathcal{I}^u_{-} = \setof{[-D, -\frac{D}{2}[, [-\frac{D}{2}, -\frac{D}{4}[ , [-\frac{D}{4}, -\frac{D}{8}[, \ldots }$ constructed in a way that the ratio between upper and lower endpoint is 2. This construction together with Lemma \ref{lemma_probs_vs_values_sum} that links the transitions, values and expected number of visits in \episodes{} of $D$ \rounds{} allows us to remove a factor of $\sqrt{D}$ compared to \UCRL{}. The results in this Lemma \ref{lemma:delta_tilde} is based on a relation between $N_k(s)$, the number of visits in the true but unknown MDP to the expected number of visits in \episodes{} of $D$ in the optimistic MDP (Lemma \ref{lemma_bound_on_visits}). 

In Lemma \ref{lemma:delta_p}, we show how to bound the $\tilde{\Delta}_{k}^{{p}}(s)$ part. The key idea is to use Bernstein based martingales concentrations inequalities instead of standard martingales. However, the adaptation was not trivial since we had to carefully introduce $\tilde{p}$ instead of the $p$ inside $\tilde{\Delta}_{k}^{{p}}(s)$. The key step is to avoid relating those two quantities through concentration inequalities. Instead we used established lemma related to the convergence of extended value iteration (Section \ref{sec:app_evi}).

\paragraph*{}
Another important aspect of our proof is that we avoid needing all $2^S$ constraints to hold \emph{with high probability} by using two tricks. The first trick is the definition of the plausible sets of MDPs $\mathcal{M}_k$ (See \cref{eq:effictive_plausible_p,eq:effictive_plausible_r}) which is the one effectively used in the proof. Note that given $(s,a)$, $\mathcal{M}_k$ only up to $S+1$ constraints and we show that the extended value iteration always converges (with probability 1. i.e the convergence does not depend on any constraint failing or holding). Furthermore, we show that the value of the policy obtained using the extended value iteration is in fact close to the optimal value for $\mathcal{M}_k$.

The second trick is that for a given $(s,a)$ our proof only requires the transitions of the true MDP to be \emph{with high probability} inside the corresponding set of transitions of $\mathcal{M}_k$ (so at most $S+1$ constraints need to hold and not $2^S$). On top of that, for a given $s$, our proof only need an additional $S+1$ constraints to hold with high probability. In particular, we only need the constraints defined by the subsets in $\mathcal{S}^+(s) \cup \mathcal{S}^-(s)$ (See \cref{eq:set_subset_whp_needed_minus,eq:set_subset_whp_needed_plus}). And we observe that the cardinality of $\mathcal{S}^+(s) \cup \mathcal{S}^-(s)$ is less than $S$ since there are at most $S$ next-states. 

\paragraph{Definition and Notations}
For any episode $k$, let's $\mathcal{M}_k$ the set of MDPs with transitions $\tilde{p}$ and rewards $\tilde{r}$ that satisfy:

\begin{align}
\tilde{r}(s,a) &\leq \upperb{r}(s,a) \; \forall (s,a)\label{eq:effictive_plausible_r}\\
\tilde{p}(\mathcal{S}_1^j | s, a) &\leq \upperb{p}(\mathcal{S}_1^j | s, a)\; \forall j \in \setof{1, \ldots S}, (s,a)\label{eq:effictive_plausible_p}
\end{align}
		
with $\mathcal{S}_{1}^{j} = \{s'_{1}, s'_{2}, \ldots s'_{j}\}\; \forall j \geq 1$ where $s'_1, s'_2, \ldots s'_S$ are such that $u_i(s'_1) \geq u_i(s'_2) \ldots \geq u_i(s'_S)$ (so the set of states sorted in descending order of their value) with $u_i$ the value at the iteration where the extended value iteration converges.

Given an interval $W \in  \mathcal{I}_-^u \cup \mathcal{I}_+^u$ and a state $s \in \mathcal{S}$, let's $\mathcal{S}^u_W(s) $ be  the set of states $s'$ such that $u_i(s')-u_i(s) \in W$.
For any state $s$, let $\mathcal{S}^{\tilde{p}}_-(s)$ contains all states $s'$ with $\tilde{p}(s'|s) - p(s'|s) < 0$ and $\mathcal{S}^{\tilde{p}}_+(s)$ contains all states $s'$ with $\tilde{p}(s'|s) - p(s'|s) > 0$ where $\tilde{p}(. | s) = \tilde{p}(. | s, \tilde{\pi}_k(s))$ and similarly for $p(. | s)$. Let us define $\mathcal{S}^{\tilde{p}_+}_W(s) = \mathcal{S}^u_W(s) \cap \mathcal{S}^{\tilde{p}}_+(s)$ and $\mathcal{S}^{\tilde{p}_-}_W(s) = \mathcal{S}^u_W(s) \cap \mathcal{S}^{\tilde{p}}_-(s)$. For any state $s$, let's define the set of subset of states $\mathcal{S}^+(s)$, $\mathcal{S}^-(s)$ as follows:

\begin{align}
\mathcal{S}^+(s) &= \{\mathcal{S}^{\tilde{p}_+}_W(s), \ldots \} \quad \forall W \in \mathcal{I}^u_+ \mid \mathcal{S}^{\tilde{p}_+}_W(s) \ne \emptyset \label{eq:set_subset_whp_needed_plus}\\
\mathcal{S}^-(s) &=\{\mathcal{S}^{\tilde{p}_-}_W(s), \ldots \} \quad \forall W \in \mathcal{I}^u_- \mid \mathcal{S}^{\tilde{p}_-}_W(s) \ne \emptyset \label{eq:set_subset_whp_needed_minus}
\end{align}

\paragraph{Detailed Proof}
We first provide the proof by only considering episodes $k$ satisfying all the followings:
\begin{enumerate}[label=\textbf{A.\arabic*},ref=A.\arabic*]

\item $N_k \geq \max\left\{2, 32D\max\curly*{C_1^{\delta_p^m}, C_2^{\delta_p^m} } \right\}$\label{assumption:N_k}

\item $\tilde{V}(\tilde{\pi}_k) \geq V^* - \epsilon_{t_k}$ where $V^*$ is the value of the optimal policy in the true but unknown MDP $M$. And $\tilde{V}(\tilde{\pi}_k)$ is the value of the policy $\tilde{\pi}_k$ returned by extended value iteration in the MDP $\tilde{M}_k$ (the MDP with transitions and rewards as in the iteration $i$ where the extended value iteration converges). \label{assumption:optimistic}

\item The true model $M$ is inside the set $\mathcal{M}_k$. Furthermore:
\begin{align}
\bar{r}(s,\tilde{\pi}_k(s)) - \E r(s, \tilde{\pi}_k(s)) &\leq C(1/2,C^{\delta^k_r}_1,C_2^{\delta^k_r}) \quad \forall s\\
\bar{p}(\mathcal{S}_c | s, \tilde{\pi}_k(s)) - p(\mathcal{S}_c | s, \tilde{\pi}_k(s)) &\leq C(p(\mathcal{S}_c | s,\tilde{\pi}_k(s)),C_1^{\delta^k_p},C_2^{\delta^k_p}) \quad \forall s, \mathcal{S}_c \in \mathcal{S}^+(s)\\
p(\mathcal{S}_c | s,\tilde{\pi}_k(s)) - \bar{p}(\mathcal{S}_c | s,\tilde{\pi}_k(s))&\leq C(p(\mathcal{S}_c | s,\tilde{\pi}_k(s)),C_1^{\delta^k_p},C_2^{\delta^k_p}) \quad \forall s, \mathcal{S}_c \in \mathcal{S}^-(s)
\end{align}\label{assumption:plausible_set}
where $\mathcal{S}^+(s)$ and $\mathcal{S}^-(s)$ are defined in \cref{eq:set_subset_whp_needed_minus,eq:set_subset_whp_needed_plus}.
\end{enumerate}

Later on, in \Cref{sec:probability_failure}, we show that \ref{assumption:plausible_set} and \ref{assumption:optimistic} hold with high probability.

Regarding \ref{assumption:N_k}, the maximum regret we can incur due to episodes not satisfying \ref{assumption:N_k} is just $\max\left\{2m, 32Dm\max\curly*{C_1^{\delta_p^m}, C_2^{\delta_p^m} } \right\}$ which we add to get the final bound.

\begin{replemma}{lemma:regret_decompose}[Regret decomposition]
		If the true model $M$ is within our plausible set $\mathcal{M}_k$ for each episode $k$, then with probability at least $1-\delta'$, 
	\begin{align}
	\regret(T) &\leq 2C_2^{\delta_r^m}m + \sqrt{C(\delta')T} + C_3 \sqrt{C_1^{\delta_r^m}SAT} + \sum_{k=1}^{m} \sum_{s}\Delta_{k}^{\tilde{p}}(s) + \sum_{k=1}^{m} \sum_{s}\Delta_{k}^{p}(s),
	\end{align}	
	where
	\begin{align}
	\tilde{\Delta}_{k}^{\tilde{p}}(s) &= N_k(s) \cdot\sum_{s'} \paren*{\tilde{p}(s'|s)-p(s'|s)} \biggl(u_i(s') - u_i(s)\biggr),\\
	\tilde{\Delta}_{k}^{p}(s)&= N_k(s) \cdot \paren*{\sum_{s'}p(s'|s)u_i(s') -u_i(s)},\\
	C_3 &= 2\paren*{\sqrt{2} +1} \\
	C(\delta') &= \ln (1/\delta'),\nonumber 
	\end{align}
\end{replemma}

%
%
%

\begin{proof} By definition of regret, we get
\[\regret(T) \defn \sum_{t=1}^{T} \left(\gain^* - r(s_t, a_t)\right)\]

\textbf{Step 1: Concentrating rewards around expected rewards.} Using Hoeffding bound similarly to Section 4.1 in \citep{jaksch2010near}, we conclude that with probability at least $1-\delta'$, the regret is:
\begin{align}
\regret(T) &\leq \sqrt{T \ln{(1/\delta')}/2} + \sum_{k=1}^{m} \sum_{s,a} N_k(s,a) \paren*{V^*-\E r(s,a)}.\label{eq:14}
\end{align}

\textbf{Step 2: Applying the convergence of Algorithm~\ref{algo:extended_vi}.} By Theorem \ref{thm:evi_convergence}, the optimistic policy $\tilde{\pi}_k$ computed by the extended value iteration at the beginning of each \episode{} $k$ in Algorithm \ref{algo:ucrl_bernstein} satisfies (since the true model in inside our plausible set):  $\gainOpt_k = \gainOpt(\piOpt_k) \geq V^* - \epsilon_{t_k}$. We have:

\begin{align}
\regret(T) &\leq \sqrt{T \ln{(1/\delta')}/2} + \sum_{k=1}^{m} \sum_{s,a} N_k(s,a) \paren*{\gainOpt_k-\E r(s,a) + \epsilon_{t_k}} \label{eq:regret}
\end{align}

\textbf{Step 3: Concentrating expected rewards to optimistic rewards.} Let's denote $\Delta_k = \sum_{s,a} N_k(s,a) \paren*{\gainOpt_k-\E r(s,a) + \epsilon_{t_k}}$

We have:

\begin{align}
\Delta_k &= \sum_{s,a} N_k(s,a) \paren*{\gainOpt_k- \rOpt(s,a)} +\sum_{s,a} N_k(s,a) \paren*{\rOpt(s,a)-\E r(s,a)} +\sum_{s,a} N_k(s,a) \epsilon_{t_k}\\
&\leq \sum_{s,a} N_k(s,a) \paren*{\gainOpt_k- \rOpt(s,a)} +\sum_{s,a} N_k(s,a) \paren*{2C(1/2,C_1^{\delta_r^m},C_2^{\delta_r^m}) + \epsilon_{t_k}}.\label{eq:delta_k}
\end{align}

\Cref{eq:delta_k} comes because $\rOpt(s,a)-\E r(s,a) = \paren*{\rOpt(s,a)-\bar{r}(s,a)} + \paren*{\bar{r}(s,a)-\E r(s,a)}$. The first term is bounded by $C(1/2,C_1^{\delta_r^k},C_2^{\delta_r^k}) \leq C(1/2,C_1^{\delta_r^m},C_2^{\delta_r^m})$ by construction (\ref{req:r_opt_bound}). The second term  is bounded due to \ref{assumption:plausible_set}. 

\textbf{Step 4: Decomposing the regret for optimistic MDP.} Letting $\tilde{\Delta}_{k} = \sum_{s,a} N_k(s,a) \paren*{\gainOpt_k- \rOpt(s,a)}$, and using the fact that, when the extended value iteration converges at iteration $i$, $\abs{u_{i+1}(s)-u_i(s)- \gainOpt_k} \leq \epsilon_{t_k}$ (By Theorem \ref{thm:evi_convergence})

\begin{align}
	\tilde{\Delta}_{k} &= \sum_{s,a} N_k(s,a) \paren*{\gainOpt_k- \rOpt(s,a)}\\
	&\leq \sum_{s,a} N_k(s,a) \paren*{u_{i+1}(s)-u_i(s) + \epsilon_{t_k} - \rOpt(s,a)}\\
	&= \sum_{s} N_k(s,\tilde{\pi}_k(s)) \biggl(u_{i+1}(s)-u_i(s) + \epsilon_{t_k} - \rOpt(s,\tilde{\pi}_k(s))\biggr)\label{step_greedy}\\
	&= \sum_{s} N_k(s,\tilde{\pi}_k(s)) \biggl(\sum_{s'} \tilde{p}(s'|s,\tilde{\pi}_k(s)) u_i(s') -u_i(s)\biggr) + \sum_{s} N_k(s,\tilde{\pi}_k(s))  \epsilon_{t_k}\label{eq:r_opt_m}
\end{align}
\eqref{step_greedy} comes from the fact that $\tilde{\pi}_k$ is a greedy policy and as a result $N_k(s,a) = 0$ for $a \neq \tilde{\pi}_k(s)$.

Also since $\tilde{\pi}_k$ is a greedy policy we will remove dependency on the action to designate probabilities. So for example we have $p(s'|s) = p(s'|s,\tilde{\pi}_k(s))$, $\tilde{p}(s'|s) = \tilde{p}(s'|s,\tilde{\pi}_k(s))$ and $N_k(s) = N_k(s, \tilde{\pi}_k(s))$. Denoting $\tilde{\Delta}_{k}(s) = N_k(s) \paren*{\sum_{s'} \tilde{p}(s'|s) u_i(s') -u_i(s)}$, we have:

\begin{align}
	\tilde{\Delta}_{k}(s) &= N_k(s) \paren*{\sum_{s'} \tilde{p}(s'|s) u_i(s') -u_i(s)}\notag\\
	&= N_k(s)  \paren*{\sum_{s'} \tilde{p}(s'|s) u_i(s') -u_i(s)}\notag\\
	&=  N_k(s)  \paren*{\sum_{s'} \paren*{\tilde{p}(s'|s)-p(s'|s)} u_i(s') + \sum_{s'}p(s'|s)u_i(s') -u_i(s)}\notag\\
	&=  N_k(s)  \paren*{\sum_{s'} \paren*{\tilde{p}(s'|s)-p(s'|s)} (u_i(s') - u_i(s))  + \sum_{s'}p(s'|s)u_i(s') -u_i(s)} \label{delta_1_s_bound}\\
	&= \Delta_{k}^{\tilde{p}}(s) + \Delta_{k}^{p}(s)\label{eq:23}
\end{align}
\eqref{delta_1_s_bound} comes from the fact that $\sum_{s'} \paren*{\tilde{p}(s'|s)-p(s'|s)} = 0$.

\textbf{Step 5: Bounding the terms due to the approximation in value iteration (last term in \eqref{eq:r_opt_m}).}

\begin{align}
\sum_{k=1}^m \sum_{s} N_k(s,\tilde{\pi}_k(s))  \epsilon_{t_k} &= \sum_{k=1}^m \epsilon_{t_k} N_k\\
&=  \sum_{k=1}^m \frac{N_k}{\sqrt{t_k}}\\
&\leq (\sqrt{2} + 1) \sqrt{T}\label{eq:24}
\end{align}
Using the fact $t_k = \max\{1, \sum_{i=1}^k N_i\}$ and $0 \leq N_k \leq t_{k-1}$ \eqref{eq:24} comes directly from C.3 in \citet{jaksch2010near}.

\textbf{Step 6: Bounding the terms due to the concentration bound on rewards (last term in \eqref{eq:delta_k}).}
\begin{align}
\sum_{k=1}^m \sum_{s,a} N_k(s,a) 2C(1/2,C_1^{\delta_r^m},C_2^{\delta_r^m}) &= 2 \sqrt{C_1^{\delta_r^m}}\sum_{k=1}^m \sum_{s,a} \frac{N_k(s,a)}{\sqrt{N_{t_k}(s,a)}} + 2C_2^{\delta_r^m} \sum_{k=1}^m \sum_{s,a} \frac{N_k(s,a)}{N_{t_k}(s,a)}\\
&\leq 2 \sqrt{C_1^{\delta_r^m}}\sum_{k=1}^m \sum_{s,a} \frac{N_k(s,a)}{\sqrt{N_{t_k}(s,a)}} + 2C_2^{\delta_r^m} m \label{eq:28}\\
&\leq 2(\sqrt{2}+1) \sqrt{C_1^{\delta_r^m}} \sqrt{SAT} + 2C_2^{\delta_r^m} m\label{eq:29}
\end{align}
We obtain \eqref{eq:28} from the extended doubling trick.
We get \eqref{eq:29} from the fact that $N_{t_k}(s,a) = \max\{1, \sum_{i=1}^k N_i(s,a)\}$ and Lemma 3 of~\citep{jaksch2010near}.

\textbf{Summary:} Equations~\ref{eq:14},~\ref{eq:23},~\ref{eq:24} and~\ref{eq:29} completes the proof of this lemma.
\end{proof}

\begin{lemma}[Bounding the effect of Optimistic MDP]
\label{lemma:delta_tilde}
		If the true model $M$ is within our plausible set $\mathcal{M}_k$ for each episode $k$ and the number of episodes is upper bounded by $m$, then,
	\[\sum_{k=1}^{m} \sum_{s}\Delta_{k}^{\tilde{p}}(s) \leq  288 \cdot \sqrt{mTC_1^{\delta_p}D \min\{\log_2\paren*{D + 1}, S\}} +  8DmC^{\delta_p}_2 \]
\end{lemma}
\begin{proof}

\ucrlvproofstep{Subdivision of the range of all possible values $u_i(s') - u_i(s)$ into sets.}
Let's consider the infinite set of non-overlapping intervals with non-negative endpoints $\mathcal{I}^u_{+} = \setof{]\frac{D}{2}, D], ]\frac{D}{4}, \frac{D}{2}] , ]\frac{D}{8}, \frac{D}{4}], \ldots }$ constructed in a way that the ratio between upper and lower endpoint is 2. Similarly, let's consider the infinite set of non-overlapping intervals with non-positive endpoints $\mathcal{I}^u_{-} = \setof{[-D, -\frac{D}{2}[, [-\frac{D}{2}, -\frac{D}{4}[ , [-\frac{D}{4}, -\frac{D}{8}[, \ldots }$. 

By Lemma~\ref{lemma:opt_diameter}, for a given state $s$, we can assign each state $s' \in \mathcal{S}$ with $u_i(s')-u_i(s) > 0$ to a unique interval $]v_1, v_2] \in \mathcal{I}_+^u$  such that $v_1 < u_i(s')-u_i(s) \leq v_2$. Similarly, for the given state $s$, we can assign each state $s'$ with $u_i(s')-u_i(s) < 0$ to a unique interval $[v_1, v_2[ \in \mathcal{I}_-^u$ such that $v_1 \leq u_i(s')-u_i(s) < v_2$.  Given an interval $W \in  \mathcal{I}_-^u \cup \mathcal{I}_+^u$ and a state $s \in \mathcal{S}$, let's $\mathcal{S}^u_W(s) $be  the set of states $s'$ such that $u_i(s')-u_i(s) \in W$.

\ucrlvproofstep{Decomposing $\tilde{\Delta}_{k}^{\tilde{p}}(s)$ in the subdivided ranges.}
Let $\mathcal{S}^{\tilde{p}}_-(s)$ contains all states $s'$ with $\tilde{p}(s'|s) - p(s'|s) < 0$ and $\mathcal{S}^{\tilde{p}}_+(s)$ contains all states $s'$ with $\tilde{p}(s'|s) - p(s'|s) > 0$ for any given state $s$. Let us define $\mathcal{S}^{\tilde{p}_+}_W(s) = \mathcal{S}^u_W(s) \cap \mathcal{S}^{\tilde{p}}_+(s)$ and $\mathcal{S}^{\tilde{p}_-}_W(s) = \mathcal{S}^u_W(s) \cap \mathcal{S}^{\tilde{p}}_-(s)$.

We then conclude that:
\begin{align}
\begin{split}
\frac{\tilde{\Delta}_{k}^{\tilde{p}}(s)}{N_k(s)} &\leq \sum_{W \in \mathcal{I}_+^u} \sum_{s' \in\mathcal{S}^{\tilde{p}_+}_W(s)} \paren*{\tilde{p}(s'|s)-p(s'|s)} \biggl(u_i(s') - u_i(s)\biggr) + \notag \\
&\quad\sum_{W \in \mathcal{I}_-^u} \sum_{s' \in \mathcal{S}^{\tilde{p}_-}_W(s)} \paren*{\tilde{p}(s'|s)-p(s'|s)} \biggl(u_i(s') - u_i(s)\biggr)\\
&\leq \sum_{W \in \mathcal{I}_+^u} \paren*{\tilde{p}(\mathcal{S}^{\tilde{p}_+}_W(s)|s)-p(\mathcal{S}^{\tilde{p}_+}_W(s)|s)} \max_{s' \in \mathcal{S}^{\tilde{p}_+}_W(s)} \biggl(u_i(s') - u_i(s)\biggr) + \notag \\
&\quad \sum_{W \in \mathcal{I}_-^u} \paren*{\tilde{p}(\mathcal{S}^{\tilde{p}_-}_W(s)|s)-p(\mathcal{S}^{\tilde{p}_-}_W(s)|s)} \min_{s' \in \mathcal{S}^{\tilde{p}_-}_W(s)} \biggl(u_i(s') - u_i(s)\biggr).
\end{split}
\end{align}
Let us focus on the positive ones for now as the arguments for the negative one follow similarly. Let $\Delta_k^{\tilde{p}_+}(s) = N_k(s)\sum_{W \in \mathcal{I}_+^u} \paren*{\tilde{p}(\mathcal{S}^{\tilde{p}_+}_W(s)|s)-p(\mathcal{S}^{\tilde{p}_+}_W(s)|s)} \max_{s' \in \mathcal{S}^{\tilde{p}_+}_W(s)} \biggl(u_i(s') - u_i(s)\biggr)$.  We have using \ref{assumption:plausible_set} and \ref{req:p_opt_bound}:

\begin{align}
\Delta_k^{\tilde{p}_+}(s)
&\leq 2N_k(s)\sum_{W \in \mathcal{I}_+^u} \paren*{c_1(p(\mathcal{S}^{\tilde{p}_+}_W(s)|s),C_1^{\delta_p}) + c_2(p(\mathcal{S}^{\tilde{p}_+}_W(s)|s),C_2^{\delta_p}) } \max_{s' \in \mathcal{S}^{\tilde{p}_+}_W(s)} \biggl(u_i(s') - u_i(s)\biggr)\\
&= 2\Delta_{k,1}^{\tilde{p}_+}(s) + 2\Delta_{k,2}^{\tilde{p}_+}(s).  
\end{align}



\ucrlvproofstep{Bounding $\Delta_{k,1}^{\tilde{p}_+}(s)$ for an episode $k$.}
\begin{align}
\Delta_{k,1}^{\tilde{p}_+}(s) &= N_k(s)\sum_{W \in \mathcal{I}_+^u} \paren*{c_1(p(\mathcal{S}^{\tilde{p}_+}_W(s)|s),C_1^{\delta_p})} \max_{s' \in \mathcal{S}^{\tilde{p}_+}_W(s)} \biggl(u_i(s') - u_i(s)\biggr)\\
&= N_k(s)\sum_{W \in \mathcal{I}_+^u} \paren*{\sqrt{\frac{C_1^{\delta_p}p(\mathcal{S}^{\tilde{p}_+}_W(s)|s) (1-p(\mathcal{S}^{\tilde{p}_+}_W(s)|s))}{N_{t_k}(s)}}} \max_{s' \in \mathcal{S}^{\tilde{p}_+}_W(s)} \biggl(u_i(s') - u_i(s)\biggr)\label{eq:33}\\
&\leq N_k(s)\sum_{W \in \mathcal{I}_+^u} \paren*{\sqrt{\frac{C_1^{\delta_p}\tilde{p}(\mathcal{S}^{\tilde{p}_+}_W(s)|s)}{N_{t_k}(s)}}} \max_{s' \in \mathcal{S}^{\tilde{p}_+}_W(s)} \biggl(u_i(s') - u_i(s)\biggr)\label{eq:34}\\
&= \sqrt{\frac{N_k(s)}{N_{t_k}(s)}}\sum_{W \in \mathcal{I}_+^u} \paren*{\sqrt{C_1^{\delta_p}N_k(s) \cdot \tilde{p}(\mathcal{S}^{\tilde{p}_+}_W(s)|s)  \left(\max_{s' \in \mathcal{S}^{\tilde{p}_+}_W(s)} \biggl(u_i(s') - u_i(s)\biggr)\right)^2 }}\label{eq:use_n-k-s}
\end{align}
\eqref{eq:33} is by the definition of $c_1$. \eqref{eq:34} is due to the fact that for all $\mathcal{S}^{\tilde{p}_+}_W(s)$, $\tilde{p}(\mathcal{S}^{\tilde{p}_+}_W(s)|s) > p(\mathcal{S}^{\tilde{p}_+}_W(s)|s)$ and $1-p(\mathcal{S}^{\tilde{p}_+}_W(s)|s) \leq 1$.

\ucrlvproofstep{Bounding the sum of $\Delta_{k,1}^{\tilde{p}_+}(s)$ over all states.}

We have from Equation~\eqref{eq:use_n-k-s}:

\begin{align}
\LHS &= \sum_{s}\Delta_{k,1}^{\tilde{p}_+}(s)\\
&\leq \sum_{s} \sqrt{\frac{N_k(s)}{N_{t_k}(s)}}\sum_{W \in \mathcal{I}_+^u} \paren*{\sqrt{C_1^{\delta_p}N_k(s) \cdot \tilde{p}(\mathcal{S}^{\tilde{p}_+}_W(s)|s)  \left(\max_{s' \in \mathcal{S}^{\tilde{p}_+}_W(s)} \biggl(u_i(s') - u_i(s)\biggr)\right)^2 }}\\
 &=\sum_{W \in \mathcal{I}_+^u} \sum_{s} \sqrt{\frac{N_k(s)}{N_{t_k}(s)}}  \sqrt{C_1^{\delta_p}N_k(s) \cdot \tilde{p}(\mathcal{S}^{\tilde{p}_+}_W(s)|s)  \left(\max_{s' \in \mathcal{S}^{\tilde{p}_+}_W(s)} \biggl(u_i(s') - u_i(s)\biggr)\right)^2 }\\
 &\leq \sum_{W \in \mathcal{I}_+^u} \sqrt{\sum_{s} \frac{N_k(s)}{N_{t_k}(s)}} \sqrt{\sum_{s} C_1^{\delta_p}N_k(s) \cdot \tilde{p}(\mathcal{S}^{\tilde{p}_+}_W(s)|s)  \left(\max_{s' \in \mathcal{S}^{\tilde{p}_+}_W(s)} \biggl(u_i(s') - u_i(s)\biggr)\right)^2}\label{eq:holder_s_s}\\
 &\leq \sqrt{C_1^{\delta_p}}\sum_{W \in \mathcal{I}_+^u} \sqrt{\sum_{s} N_k(s) \cdot \tilde{p}(\mathcal{S}^{\tilde{p}_+}_W(s)|s)  \left(\max_{s' \in \mathcal{S}^{\tilde{p}_+}_W(s)} \biggl(u_i(s') - u_i(s)\biggr)\right)^2}\label{eq:before_merge_s}
\end{align}
\eqref{eq:holder_s_s} is obtained by applying H\"older's inequality over $s$.
\eqref{eq:before_merge_s} comes from the extended doubling trick.

We then construct the set of intervals $\mathcal{I}^p = \{ ]\frac{1}{2}, 1], ]\frac{1}{4}, \frac{1}{2}], \ldots ]\frac{1}{D}, \frac{2}{D}], \ldots \}$. We will sum together states whose $\tilde{p}(\mathcal{S}^{\tilde{p}_+}_W(s)|s)$ belongs to the same interval in $\mathcal{I}^p$. Given an interval $W_p \in \mathcal{I}^p$, let's call $\mathcal{S}_{W_p}$ the set of all states such that $s \in \mathcal{S}_{W_p}$ if $\tilde{p}(\mathcal{S}^{\tilde{p}_+}_W(s)|s) \in W_p$.

We will denote by $\mathcal{I}^p_<$ the set  $\{ ]\frac{1}{2}, 1], ]\frac{1}{4}, \frac{1}{2}], \ldots ]\frac{1}{D}, \ldots \frac{2}{D}]  \}$ and $\mathcal{I}^p_>$ the complement of set $\mathcal{I}^p_<$

Continuing from \eqref{eq:before_merge_s} and letting $W = ]\lowerb{b}(W), 2\lowerb{b}(W)]$ for any $W \in \mathcal{I}_+^u$, we have:

\begin{align}
\LHS &=\frac{\sum_{s}\Delta_{k,1}^{\tilde{p}_+}(s)}{\sqrt{C_1^{\delta_p}}}\\
&\leq \sum_{W \in \mathcal{I}_+^u} \sqrt{\sum_{s} N_k(s) \cdot \tilde{p}(\mathcal{S}^{\tilde{p}_+}_W(s)|s)  \left(\max_{s' \in \mathcal{S}^{\tilde{p}_+}_W(s)} \biggl(u_i(s') - u_i(s)\biggr)\right)^2}\label{eq:avoid_logD}\\
&\leq \sum_{W \in \mathcal{I}_+^u} \sqrt{\sum_{W_p \in \mathcal{I}^p}N_k(\mathcal{S}_{W_p}) \cdot \max_{s \in \mathcal{S}_{W_p}} \tilde{p}(\mathcal{S}^{\tilde{p}_+}_W(s)|s)   \left(\max_{s' \in \mathcal{S}^{\tilde{p}_+}_W(s)} \biggl(u_i(s') - u_i(s)\biggr)\right)^2}\\
&\leq \sum_{W \in \mathcal{I}_+^u} \sqrt{\sum_{W_p \in \mathcal{I}^p_<}N_k(\mathcal{S}_{W_p}) \cdot \max_{s \in \mathcal{S}_{W_p}} \tilde{p}(\mathcal{S}^{\tilde{p}_+}_W(s)|s)   \left(\max_{s' \in \mathcal{S}^{\tilde{p}_+}_W(s)} \biggl(u_i(s') - u_i(s)\biggr)\right)^2}\notag\\
&\quad\quad + \sum_{W \in \mathcal{I}_+^u} \sqrt{\sum_{W_p \in \mathcal{I}^p_>}N_k(\mathcal{S}_{W_p}) \cdot \max_{s \in \mathcal{S}_{W_p}} \tilde{p}(\mathcal{S}^{\tilde{p}_+}_W(s)|s)   \left(\max_{s' \in \mathcal{S}^{\tilde{p}_+}_W(s)} \biggl(u_i(s') - u_i(s)\biggr)\right)^2}\\
&\leq \sum_{W \in \mathcal{I}_+^u} \sqrt{\sum_{W_p \in \mathcal{I}^p_<}N_k(\mathcal{S}_{W_p}) \cdot \max_{s \in \mathcal{S}_{W_p}} \tilde{p}(\mathcal{S}^{\tilde{p}_+}_W(s)|s)   \left(\max_{s' \in \mathcal{S}^{\tilde{p}_+}_W(s)} \biggl(u_i(s') - u_i(s)\biggr)\right)^2}\notag\\
&\quad\quad + \sum_{k=0}^{\infty} \sqrt{\sum_{j=0}^{\infty}N_k(\mathcal{S}_{W_p}) \cdot \frac{1}{2^j}  \cdot \frac{D}{2^k}}\\
&\leq \sum_{W \in \mathcal{I}_+^u} \sqrt{\sum_{W_p \in \mathcal{I}^p_<}N_k(\mathcal{S}_{W_p}) \cdot \max_{s \in \mathcal{S}_{W_p}} \tilde{p}(\mathcal{S}^{\tilde{p}_+}_W(s)|s)   \left(\max_{s' \in \mathcal{S}^{\tilde{p}_+}_W(s)} \biggl(u_i(s') - u_i(s)\biggr)\right)^2}\\
&\quad\quad +  (2\sqrt{2}+2) \cdot \sqrt{ N_k D}\notag\\
&\leq \sum_{W \in \mathcal{I}_+^u} \sqrt{\sum_{W_p \in \mathcal{I}^p_<} 384 N_k \cdot   \left(2\lowerb{b}(W)\right)} +  (2\sqrt{2}+2) \cdot \sqrt{N_k D}\label{eq:actual_use_n-k-s}\\
&\leq \sum_{W \in \mathcal{I}_+^u} \sqrt{768 \cdot N_k \cdot \lowerb{b}(W) \cdot \min\{\log_2\paren*{D + 1}, S\}} + (2\sqrt{2}+2) \cdot \sqrt{N_k D}\label{eq:actual_apply_u_times_u}\\
&\leq \sqrt{N_k D \min\{\log_2\paren*{D + 1}, S\}} \sqrt{384} (2 + \sqrt{2}) + (2\sqrt{2}+2) \cdot \sqrt{N_k D}\\
&\leq \sqrt{N_k D} \cdot \paren*{\sqrt{\min\{\log_2\paren*{D + 1}, S\}} \cdot (2\sqrt{384} + \sqrt{768}) + 2\sqrt{2} + 2}\\
&\leq 72\sqrt{N_k D \min\{\log_2\paren*{D + 1}, S\}}\label{eq:sum_over_all_s_1}
\end{align}
\eqref{eq:actual_use_n-k-s} comes by applying Lemma \ref{ucrlv:lemma:num_groups}

\ucrlvproofstep{Bounding the sum of $\Delta_{k,1}^{\tilde{p}_+}(s)$ over all episodes and states.}

We have from Equation~\eqref{eq:sum_over_all_s_1}:
\begin{align}
\sum_{k=1}^{m} \sum_{s}\Delta_{k,1}^{\tilde{p}_+}(s)&\leq \sum_{k=1}^{m} 72 \sqrt{C_1^{\delta_p} N_k D \min\{\log_2\paren*{D + 1}, S\}}\\
&\leq 72 \cdot \sqrt{m} \sqrt{TC_1^{\delta_p}D \min\{\log_2\paren*{D + 1}, S\}}\label{eq:first_term}
\end{align}
\eqref{eq:first_term} is obtained by applying H\"older's inequality over $k$.

\ucrlvproofstep{Bounding and summing $\Delta_{k,2}^{\tilde{p}_+}(s)$ over all episodes and states.}

\begin{align}
\sum_{k=1}^{m} \sum_{s}\Delta_{k,2}^{\tilde{p}_+}(s)&=\sum_{k=1}^{m} \sum_{s} N_k(s)\sum_{W \in \mathcal{I}_+^u} \paren*{c_2(p(\mathcal{S}^{\tilde{p}_+}_W(s)|s),C_2^{\delta_p})} \max_{s' \in \mathcal{S}^{\tilde{p}_+}_W(s)} \biggl(u_i(s') - u_i(s)\biggr)\\
&= \sum_{k=1}^{m} \sum_{s} N_k(s)\sum_{W \in \mathcal{I}_+^u} \paren*{\frac{C_2^{\delta_p}}{N_{t_k}(s)}} \max_{s' \in \mathcal{S}^{\tilde{p}_+}_W(s)} \biggl(u_i(s') - u_i(s)\biggr)\label{eq:46}\\
&\leq \sum_{k=1}^{m} \sum_{s} \frac{C_2^{\delta_p} N_k(s)}{N_{t_k}(s)}\sum_{j=0}^{\infty} \frac{D}{2^j}\label{eq:47}\\
&\leq 2DC_2^{\delta_p}\sum_{k=1}^{m} \sum_{s} \frac{N_k(s)}{N_{t_k}(s)}\label{eq:48}\\
&\leq 2DmC_2^{\delta_p}\label{eq:second_term}
\end{align}

The lemma comes by combining 
\eqref{eq:second_term} and \eqref{eq:first_term} and multiplying by 2 for a similar analysis for the states with negative $u_i(s') - u_i(s)$.
\end{proof}

\begin{lemma}[Bounding the Martingale for Original MDP]
\label{lemma:delta_p}
		If the true model $M$ is within our plausible set $\mathcal{M}_k$ for each episode $k$ and the number of episodes is upper bounded by $m$, then, with probability at least $1-\delta'$:
	\[\sum_{k=1}^{m} \sum_{s} \tilde{\Delta}_{k}^{p}(s) \leq 157 \sqrt{ \min\{\log_2\paren*{D + 1}, S\} \cdot D \cdot S \cdot T \cdot \ln \frac{1}{\delta'}} + \frac{2}{3} D \ln \frac{1}{\delta'} + Dm\]
\end{lemma}
\begin{proof}
\textbf{Step 1: Reducing the sum to a martingale.} Consider the random variable $X_t = \sum_{s'}p(s'|s_t, a_t)u_i^k(s') -u_i^k(s_{t+1})$.
For clarity using $u_i^k$ to mean the value at \episode{} $k$, we have:
\begin{align}
	\sum_{k=1}^{m} \sum_{s} \tilde{\Delta}_{k}^{p}(s)&= \sum_{k=1}^{m} \sum_{s} N_k(s) \cdot \paren*{\sum_{s'}p(s'|s)u_i^k(s') -u_i^k(s)}\\
	&= \sum_{k=1}^{m} \sum_{t=t_k}^{t_{k+1}-1} \paren*{\sum_{s'}p(s'|s_t, a_t)u_i^k(s') -u_i^k(s_{t})}\\
	&= \sum_{k=1}^{m}\paren*{ \sum_{t=t_k}^{t_{k+1}-1} \paren*{\sum_{s'}p(s'|s_t, a_t)u_i^k(s') -u_i^k(s_{t+1})}}\notag\\
	&\quad\quad + \sum_{k=1}^{m}\paren*{\sum_{s'}p(s'|s_{t_{k+1}-1}, a_{t_{k+1}-1})u_i^k(s') - u_i^k(s_{t_k})}\\
	&= \sum_{k=1}^{m}\sum_{t=t_k}^{t_{k+1}-1} X_t + \sum_{k=1}^{m}\paren*{\sum_{s'}p(s'|s_{t_{k+1}-1}, a_{t_{k+1}-1})u_i^k(s') - u_i^k(s_{t_k})}\label{eq:53}\\
	&\leq \sum_{k=1}^{m}\sum_{t=t_k}^{t_{k+1}-1} X_t + Dm\label{eq:54}
\end{align}
\eqref{eq:54} is due to Lemma \ref{lemma:opt_diameter} and the fact that the average of a set of real numbers is less than their maximum.

\paragraph{Step 2: Proving the conditional expectation of $X_t$ is $0$.}
\begin{align*}
	\E \braket*{X_t \mid s_1, a_1, \ldots s_t, a_t} &= \E \left[\sum_{s'}p(s'|s_t, a_t)u_i^k(s') -u_i^k(s_{t+1})\right]\\
	&= \E \left[\sum_{s'}p(s'|s_t, a_t)u_i^k(s') \mid s_t, a_t\right] -\E \left[u_i^k(s_{t+1}) \mid s_t, a_t\right]\\
	&= \sum_{s'}p(s'|s_t, a_t)u_i^k(s') -\sum_{s'}p(s'|s_t, a_t)u_i^k(s') = 0
\end{align*}

\paragraph{Step 3: Proving the sum of conditional expectation of $X_t^2$ is upper bounded.}
The idea is to use the Bernstein inequalities for martingales (Lemma 1 in \cite{cesa2008improved}). For that we need to bound $\E \braket*{X^2_t \mid s_1, a_1, \ldots s_t, a_t}$. To shorten notation, let's write $\E \braket*{X^2_t \mid \ldots}$ for $\E \braket*{X^2_t \mid s_1, a_1, \ldots s_t, a_t}$.
%
\begin{align}
	\sum_{t=t_k}^{t_{k+1}}  \E \braket*{X^2_t \mid \ldots} &= \sum_{t=t_k}^{t_{k+1}} \E \braket*{X^2_t \mid s_t, a_t} \\
	&= \sum_{t=t_k}^{t_{k+1}} \sum_{s} p(s | s_t, a_t) \paren*{\sum_{s'}p(s'|s_t, a_t)u_i(s') -u_i(s)}^2\\
	&= \sum_{t=t_k}^{t_{k+1}} \sum_{s} p(s | s_t) \paren*{\sum_{s'}p(s'|s_t)\paren*{u_i(s')-u_i(s_t)} +u_i(s_t)-u_i(s)}^2\\
	&\leq 2\sum_{t=t_k}^{t_{k+1}} \sum_{s} p(s | s_t) \paren*{\paren*{\sum_{s'}p(s'|s_t)\paren*{u_i(s')-u_i(s_t)}}^2 +\paren*{u_i(s_t)-u_i(s)}^2}\label{a_plus_b_squared}\\
	&\leq 2D\sum_{t=t_k}^{t_{k+1}} \sum_{s} p(s | s_t) \paren*{\abs*{\sum_{s'}p(s'|s_t)\paren*{u_i(s')-u_i(s_t)}} +\abs*{u_i(s_t)-u_i(s)}}\label{eq:60}\\
	&= 2D\sum_{t=t_k}^{t_{k+1}}  \paren*{\abs*{\sum_{s'}p(s'|s_t)\paren*{u_i(s')-u_i(s_t)}} +\sum_{s} p(s | s_t)\abs*{u_i(s_t)-u_i(s)}}\label{eq:61}\\
	&\leq 4D\sum_{t=t_k}^{t_{k+1}} \sum_{s'} p(s' | s_t)\abs*{u_i(s_t)-u_i(s)}\label{sum_p_sum_abs}\\
	&= 4D\sum_{s}N_k(s)\sum_{s'} p(s' | s)\abs*{u_i(s)-u_i(s')}\\
	&\leq 8D\sum_{s: \tilde{\Delta}_k^p(s) \geq 0}N_k(s)\sum_{s'} p(s' | s)\abs*{u_i(s)-u_i(s')}\label{sum_given_delta_positive}\\
	&\leq 16D\sum_{s: \tilde{\Delta}_k^p(s) \geq 0}N_k(s)\sum_{s': u_i(s')-u_i(s) \geq 0} \tilde{p}(s' | s)\abs*{u_i(s)-u_i(s')}\label{ucrlv_p_tilde_introduced}\\
	&\leq 16D \sum_s N_k(s) \sum_{W \in \mathcal{I}^u_+}  \tilde{p}(\mathcal{S}^{\tilde{p}_+}_W(s)|s) \cdot   \abs*{\max_{s' \in \mathcal{S}^{\tilde{p}_+}_W(s)} \biggl(u_i(s') - u_i(s)\biggr)}\\
	&\leq 24576 \cdot N_k \cdot D \cdot \min\{\log^2_2\paren*{D+1}, S^2\}\label{eq:log_martingal}\\
	&\leq 24576 \cdot N_k \cdot DS \cdot \min\{\log_2\paren*{D + 1}, S\}\label{eq:martingal}
\end{align}


\eqref{a_plus_b_squared} comes from the fact that for any two real numbers $a,b$: $(a+b)^2 \leq 2(a^2 + b^2)$

\eqref{eq:60} is due to Lemma \ref{lemma:opt_diameter}

\eqref{eq:61} comes from the fact that $\sum_s p(s|s_t) = 1$ and \eqref{sum_p_sum_abs} for any set of real numbers $a_j$, $\abs*{\sum_j a_j} \leq \sum_j \abs*{a_j}$

\eqref{sum_given_delta_positive} uses the fact that in episode $k$, if $\sum_s \tilde{\Delta}_k^p(s) < 0$, then we have a trivial bound. So we can assume $\sum_s \tilde{\Delta}_k^p(s) \geq 0$. Now for any set of real numbers $a_j$ with $X = \sum_j a_j \geq 0$, we have $\sum_{j: a_j \geq 0} \abs{a_j} + \sum_{j: a_j < 0} \abs{a_j} \leq 2 \sum_{j: a_j \geq 0} \abs{a_j}$

\eqref{ucrlv_p_tilde_introduced} comes by applying \Cref{lemma:sum_p_vs_tilde}.

\eqref{eq:log_martingal} comes similarly to the derivations following \eqref{eq:avoid_logD}. However here we obtain $\log^2_2 D$ since the difference $u_i$ is not "squared".

We can sum this over all episodes. So we have:

\begin{align}
	\sum_{k=1}^{m}\sum_{t=t_k}^{t_{k+1}} \E \braket*{X^2_t \mid s_1, a_1, \ldots, s_t, a_t}&= \sum_{k=1}^{m} 24576 \cdot N_k \cdot DS \cdot \min\{\log_2\paren*{D + 1}, S\}\\
	&= 24576 \cdot DS \cdot T \min\{\log_2\paren*{D + 1}, S\}\label{eq:69}
\end{align}

\paragraph{Step 4: Proving the martingale concentration bound.}
Plugging \eqref{eq:69} into Lemma 1 in \citet{cesa2008improved} and using the Inequality reverse Lemma (Lemma 1 in \citet{peel2010empirical}). We can conclude that with probability at least $1-\delta'$:

\[\sum_{k=1}^{m} \sum_{s} \tilde{\Delta}_{k}^{p}(s) \leq 157 \sqrt{ \min\{\log_2\paren*{D + 1}, S\} \cdot D \cdot S \cdot T \cdot \ln \frac{1}{\delta'}} + \frac{2}{3} D \ln \frac{1}{\delta'} + Dm \]

\subsection{Probability of failing confidence interval}
\label{sec:probability_failure}

\paragraph{Proving high probability of \ref{assumption:plausible_set}}
We first prove that \ref{assumption:plausible_set} holds with high probability for a fixed episode $k$.

Note that the set $\mathcal{M}_k$ by definition has been constructed for a given state $(s,a)$ using at most $S$ constraints on the transitions and $1$ constraint on the rewards. The remaining conditions needs at most $S+1$ event to holds. So A union bound over all state-actionspair lead to a union bound over at most $2SA (S+1)$ events.

\todo[inline]{You can remove a factor of $A$ in the union bound by considering $\mathcal{M}_k$ defined not for all $a$ but for the action of an optimal deterministic policy in the unknown true MDP M. \ref{assumption:optimistic} will still hold even in this case. You can do that but it won't really remove a factor of A since you still need the action played by the policy to hold; and that action may not be fixed over different episodes; so need to make all actions hold.}

\paragraph{Proving high probability of \ref{assumption:optimistic}}

We prove this for a fixed episode $k$. 

First observe that the set of MDP $\mathcal{M}'_k$ constructed using all $2^S$ constraints contains a communicating MDP. This is because for any two-pairs of states state $s,s'$, there always exists an extended action with non-zero probability from $s$ to $s'$. So the extended value iteration will converge (after a finite number of iterations) and at convergence, we have an $\epsilon$-optimal policy $\tilde{\pi}_k$ for the extended MDP constructed using $\mathcal{M}'_k$ (Theorem \ref{thm:evi_convergence}). Let $u'_i$ the value at the convergent iteration. By definition this means that the span of $u'_{i+1}-u'_i$ is less than $\epsilon$.

We will now show that $\tilde{\pi}_k$ is also an $\epsilon$-optimal policy for $\mathcal{M}_k$. To find an  $\epsilon$-optimal policy for $\mathcal{M}_k$, we can again use extended value iteration. Let's set the initial value $u_0$ to $u'_i$; so $u_0 = u'_i$. We can confirm that $u_1$ will be exactly equal to $u'_{i+1}$. So the span of $u_1-u_0$ is less than $\epsilon$ and the policy  $\tilde{\pi}_k$ is thus $\epsilon$-optimal for the extended MDP constructed using $\mathcal{M}_k$.

We had already shown (Proof of high probability for \ref{assumption:plausible_set}) that with high probability $M \in \mathcal{M}_k$. This with the fact that $\tilde{\pi}_k$ is $\epsilon$-optimal for the extended MDP constructed using $\mathcal{M}_k$ leads directly to the high probability of \ref{assumption:optimistic} and furthermore $\max_s u'_{i} - \min_s u'_i(s)  \leq D$ (Theorem \ref{thm:evi_convergence}).


\paragraph{Probability over all episodes}
The probability of failing over all episodes is derived from \cite{ucbv}(Theorem 1) and to avoid the need of knowing the horizon $T$ for scaling the confidence intervals, we compute the failure probability starting from the episode where $t \geq \sqrt{T}$ (inducing at most an extra $\sqrt{T}$ in regret).

\todo[inline]{For the rewards , you can do better than Audibert theorem and remove the additional $log T$ required. Just look at the min. However, for the probability, you can't really use Audibert result or the min trick since the subsets that needs to hold may not be the same across episodes. So better just sum over all episodes?}
\end{proof}

%
%
%
%
%
%
%

\section{Linking the Number of Visits of a State in an MDP to the Value of a Policy}

We begin by proving Lemma~\ref{lemma_probs_vs_values_sum} that is fundamental to decrease a $\sqrt{D}$ factor in the final result.
\begin{replemma}{lemma_probs_vs_values_sum}
	Let $\mathcal{S}_0$ and $\mathcal{S}_1$ any two non empty subset of states. Let $s \in \mathcal{S}_0$. We have:
	
	\[ \tilde{u}^c_y(\mathcal{S}_0 | s) \cdot \abs*{ \min_{s' \in \mathcal{S}_1}u_i(s')-u_i(s)} \min_{s' \in \mathcal{S}_0} \tilde{p}(\mathcal{S}_1 | s') \leq y   \]
	
	where $\tilde{u}^c_y(\mathcal{S}_0 | s)$ represents the total expected number of time the optimistic policy $\tilde{\pi}_k$ visits the states $s' \in \mathcal{S}_0$ when starting from state $s$ and playing for $y$ steps in the optimistic MDP $\tilde{M}_k$. $y = \min\{x, D\}$ with $x$ being the number of \rounds{} you need to play, when starting from $s$, to visit any state in $\mathcal{S}_0$ for $\frac{1}{\min_{s' \in \mathcal{S}_0} \tilde{p}(\mathcal{S}_1 | s')}$ times in expectation.
\end{replemma}

\begin{proof}
	\textbf{Part 1: Proving~\eqref{eq:direction_1}.}  We begin by proving the following direction of the statement of lemma~\ref{lemma_probs_vs_values_sum} :
	\begin{align}
	\tilde{u}^c_y(\mathcal{S}_0 | s) \cdot \paren*{ \min_{s' \in \mathcal{S}_1}u_i(s')-u_i(s)} \min_{s' \in \mathcal{S}_0} \tilde{p}(\mathcal{S}_1 | s') \leq y \label{eq:direction_1}
	\end{align}
	\textbf{Case 1: $\min_{s' \in \mathcal{S}_1}u_i(s')-u_i(s) \leq 0$.}
	If $\min_{s' \in \mathcal{S}_1}u_i(s')-u_i(s) \leq 0$, \eqref{eq:direction_1} trivially holds since $y \geq 0$.
	
	\textbf{Case 2: $\min_{s' \in \mathcal{S}_1}u_i(s')-u_i(s) > 0$.}
	$\tilde{u}^c_x(s)$ represents \emph{the total expected number of time policy $\tilde{\pi}_k$ visits state $s$ when starting from $s$ and playing for $x$-steps}.	
	By definition of $x$, we get \begin{align}\label{eqn:def_ucx}
	\tilde{u}^c_x(\mathcal{S}_0 | s) = \frac{1}{\min_{s'' \in \mathcal{S}_0} \tilde{p}(\mathcal{S}_1 | s'')}.
	\end{align}
	
	Now, we compute a lower bound on the expected number of times, $\EX z$, a policy reach at least one state in $\mathcal{S}_1$ when starting from $s$ and playing for $x$ \rounds{} in the optimistic MDP.
	\begin{align}
	\EX z &= \sum_{ s' \in \mathcal{S}_0  } \tilde{u}^c_x(s' | s) \tilde{p}(\mathcal{S}_1 | s')\\
	&= \sum_{ s' \in \mathcal{S}_0  } \frac{\tilde{u}^c_x(s' | s)}{\tilde{u}^c_x(\mathcal{S}_0 | s)} \tilde{p}(\mathcal{S}_1 | s')\tilde{u}^c_x(\mathcal{S}_0 | s)\\
	&= \sum_{ s' \in \mathcal{S}_0  } \frac{\tilde{u}^c_x(s' | s)}{\tilde{u}^c_x(\mathcal{S}_0 | s)}\frac{\tilde{p}(\mathcal{S}_1 | s')}{\min_{s'' \in \mathcal{S}_0} \tilde{p}(\mathcal{S}_1 | s'')}\\
	&\geq \sum_{ s' \in \mathcal{S}_0  } \frac{\tilde{u}^c_x(s' | s)}{\tilde{u}^c_x(\mathcal{S}_0 | s)}\\
	&= 1\label{y_is_1}
	\end{align}
	
	Let us denote \emph{the total expected $i$-step reward when starting from state $s$ and following policy $\tilde{\pi}_k$} as $u_i(s)$. 
	
	Fix any give state $s$ and a set of states $\mathcal{S}_1$. If $\ell$ is the expected number of steps that $\tilde{\pi}_k$ takes to reach a state in $\mathcal{S}_1$ from $s$ then: $$u_i(s) \geq \min_{s' \in \mathcal{S}_1} u_i(s') - \ell.$$ 
	Since for the first $\ell$ steps, we have lost at most $\ell$ rewards compared to the state with minimum value in $\mathcal{S}_1$. Using this fact with \eqref{y_is_1} and the definition of $x$, we have:
	\[ u_i(s) \geq  \min_{s' \in \mathcal{S}_1} u_i(s') - x \]
	which can be equivalently written as
	
	\begin{align}
	\min_{s' \in\mathcal{S}_1} u_i(s') -u_i(s) \leq x \label{almost_step}
	\end{align}

	By definition of $\tilde{u}^c_x(\mathcal{S}_0 | s)$ in \eqref{eqn:def_ucx}, we get  $x = \frac{x}{  \tilde{u}^c_x(\mathcal{S}_0 | s) \cdot \min_{s'' \in \mathcal{S}_0} \tilde{p}(\mathcal{S}_1 | s'')}$.
	
	Plugging this into \eqref{almost_step}, we get
	
	\[ \min_{s' \in\mathcal{S}_1} u_i(s') -u_i(s) \leq \frac{x}{  \tilde{u}^c_x(\mathcal{S}_0 | s) \cdot \min_{s'' \in \mathcal{S}_0} \tilde{p}(\mathcal{S}_1 | s'')}\]
	
	Since by assumption, $\min_{s' \in\mathcal{S}_1} u_i(s') -u_i(s) > 0$, we have:
	
	\begin{align}
	\paren*{\min_{s' \in\mathcal{S}_1} u_i(s') -u_i(s)} \cdot \min_{s'' \in \mathcal{S}_0} \tilde{p}(\mathcal{S}_1 | s'') \leq \frac{x}{ \tilde{u}^c_x(\mathcal{S}_0 | s)} \label{eq:need_sign}
	\end{align}

	Now there are two cases $x \leq D$ or $x > D$. We treat each one separately.
	
	\textbf{Case 2.1: $x \leq D$.}
	Then \eqref{eq:direction_1} comes directly from \eqref{eq:need_sign}.
	%
	
	
	
	\textbf{Case 2.2:  $x > D$.}
	The condition means that the expected number of visits will satisfy $\tilde{u}^c_D(\mathcal{S}_0 | s) \leq \tilde{u}^c_x(\mathcal{S}_0 | s)$. By definition of $\tilde{u}^c_x(\mathcal{S}_0 | s)$ in \eqref{eqn:def_ucx}, we obtain $\min_{s'' \in \mathcal{S}_0} \tilde{p}(\mathcal{S}_1 | s'') \leq \frac{1}{\tilde{u}^c_D(\mathcal{S}_0 | s)}.$
	
	Since $\min_{s' \in \mathcal{S}_1}u_i(s')-u_i(s)$ is positive, we have
	\begin{align}
	\tilde{u}^c_D(\mathcal{S}_0 | s) \cdot \paren*{ \min_{s' \in \mathcal{S}_1}u_i(s')-u_i(s)} \min_{s' \in \mathcal{S}_0} \tilde{p}(\mathcal{S}_1 | s') &\leq \tilde{u}^c_D(\mathcal{S}_0 | s) \cdot \paren*{ \min_{s' \in \mathcal{S}_1}u_i(s')-u_i(s)} \frac{1}{\tilde{u}^c_D(\mathcal{S}_0 | s)}\\
	&\leq  \min_{s' \in \mathcal{S}_1}u_i(s')-u_i(s)\\
	&\leq D \label{eq:more_time_round}
	\end{align}
	
	Hence, Part 1 (Equation~\eqref{eq:direction_1}) of lemma \ref{lemma_probs_vs_values_sum} holds true.

	\textbf{Part 2: Proving~\eqref{eq:direction_2}.} We now prove the other direction of lemma \ref{lemma_probs_vs_values_sum}.
	\begin{align}
	\tilde{u}^c_D(\mathcal{S}_0 | s) \cdot \paren*{ \min_{s' \in \mathcal{S}_1}u_i(s')-u_i(s)} \min_{s' \in \mathcal{S}_0} \tilde{p}(\mathcal{S}_1 | s') \geq -D \label{eq:direction_2}
	\end{align}
	Proof of \eqref{eq:direction_2} follows the exact steps as the one for \eqref{eq:direction_1} while accommodating the following changes:
	\begin{itemize}
		\item[i.] For any given state $s$ and a set of states $\mathcal{S}_1$, if we start from $s$ and can reach at least $s' \in \mathcal{S}_1$ for the first time after an expected $\ell$ steps; then $$u_i(s) \leq \max_{s' \in \mathcal{S}_1}u_i(s') + \ell. $$ 
		Since for the first $\ell$ steps, $s$ have gain at most $\ell$ rewards and the first state $s' \in \mathcal{S}_1$ transited to could be the one with maximum value.
		\item[ii.] Also $\max_{s' \in \mathcal{S}_1}u_i(s') = - \min_{s'  \in \mathcal{S}_1} -u_i(s')$.
	\end{itemize} 
	
	Parts 1 and 2 together complete the proof of Lemma \ref{lemma_probs_vs_values_sum}.
\end{proof}

Lemma~\ref{lemma_bound_on_visits} provides a bound on the number of visits to a subset of states in a given episode $k$.
\begin{replemma}{lemma_bound_on_visits}
Let $\mathcal{S}_0$ any subset of states, $k$ any episode in which the true MDP is inside the plausible set $\mathcal{M}_k$ such that $N_k \geq 32 \cdot y \cdot \max\{C_1^{\delta_p}, C_2^{\delta_p}\}$ for any $y$ . We have with probability at least $1-S\delta_p$:

\[N_k(\mathcal{S}_0) \leq  12 \frac{N_k}{y} \max_{s \in \mathcal{S}}\tilde{u}^c_y(\mathcal{S}_0;s)\]
where $\tilde{u}^c_y(\mathcal{S}_0;s)$ is the expected number of times the states in $\mathcal{S}_0$ are played by policy $\tilde{\pi}_k$ in $\tilde{M}_k$ after $y$ steps from an initial state $s$.
\end{replemma}

\begin{proof}
%
%

We first provide the proof when $\mathcal{S}_0$ is a single state $s$. We will extend later to any subset.

Let $\bar{u}^c_y(s'';s')$ the expected number of times $s''$ is played by policy $\tilde{\pi}_k$ in $\tilde{M}_k$ after $y$ steps from an initial state $s'$. Define $\bar{u}^c_y(s) = N_k(s) \frac{y}{N_k}$

We will now compare $\bar{u}^c_{y}(s)$ and $\tilde{u}^c_{y}(s; s)$

We can compute a bound  for $\bar{u}^c_{y}(s; s)$ by counting the states that come immediately before $s$. In particular, the number of times we reach $s$ from any $s'$ will be upper bounded by $\bar{u}^c_{y}(s'; s) \bar{p}(s|s')$ with $\bar{p}(s|s') = \frac{N_k(s | s')}{N_k(s')}$ where $N_k(s | s')$ is the number of times $s$ is played immediately after $s'$ in episode $k$. Each time we reach $s$, we will continue playing $s$ at most $\frac{1}{1-\bar{p}(s|s)}$ times. So we have:
\begin{align}
\bar{u}^c_{y}(s) &\leq \min \curly*{y,\frac{1 + \sum_{s' \ne s}\bar{u}^c_{y}(s') \bar{p}(s|s')}{1-\bar{p}(s|s)}}\label{eq:ubar}
\end{align}

\todo[inline]{Don't use $\bar{p}$ notation as it is confusing. It is not the same thing, rather the $\bar{p}$ only for the current episode and not up to the current episode.}

Similarly, we can conclude that:

\begin{align}
\tilde{u}^c_{y}(s; s) &\geq \min\curly*{y, \frac{1+ \sum_{s' \ne s}\tilde{u}^c_{y}(s'; s) \tilde{p}(s|s')}{1-\tilde{p}(s|s)}}\label{eq:utilde}
\end{align}

We would like to express $\bar{u}^c_{y}(s)$ of \eqref{eq:ubar} in term of $\tilde{u}^c_{y}(s; s)$ of \eqref{eq:utilde}. 

\paragraph{Step 1}
First let's bound the denominator of \eqref{eq:ubar} in term of the denominator of \eqref{eq:utilde}.

Case 1: First note that for any state $s$ for which $N_k(s) \leq C_1 \max\{C_1^{\delta_p}, C_2^{\delta_p}\} \tilde{u}^c_y(s;s)$,

 we have $N_k(s) \leq \frac{C_1}{C_N} \frac{N_k}{y} \tilde{u}^c_y(s;s)$.

Case 2: $\tilde{p}(\mathcal{S} \setminus s | s) < \frac{1}{2\tilde{u}^c_y(s;s)}$

Note that if $2\tilde{u}^c_y(s;s) \geq y$, this case becomes impossible since (using \eqref{eq:utilde}) it leads to $\tilde{u}^c_y(s;s) \geq 2 \tilde{u}^c_y(s;s)$. Otherwise, we have $\tilde{u}^c_y(s;s) = y$ and $N_k(s) \leq \frac{N_k}{y} \tilde{u}^c_y(s;s)$.

Case 3: $N_k(s) \geq C_1 \max\{C_1^{\delta_p}, C_2^{\delta_p}\} \tilde{u}^c_y(s;s)$ and $\tilde{p}(\mathcal{S} \setminus s | s) \geq \frac{1}{2\tilde{u}^c_y(s;s)}$,

we have
$1 - \bar{p}(s|s) = \bar{p}(\mathcal{S} \setminus s | s)$

With probability at least $1-\delta_p$, we have:
$\bar{p}(\mathcal{S} \setminus s | s) \geq \tilde{p}(\mathcal{S} \setminus s | s) - 2\sqrt{\frac{C_1^{\delta_p}\bar{p}(\mathcal{S} \setminus s | s)}{N_k(s)}} - 2 \frac{C_2^{\delta_p}}{N_k(s)}$.

Using the bound on $N_k(s)$ by assumption, we have:
$\bar{p}(\mathcal{S} \setminus s | s) \geq \tilde{p}(\mathcal{S} \setminus s | s) - \sqrt{4\frac{\bar{p}(\mathcal{S} \setminus s | s)}{C_1 \tilde{u}^c_y(s;s)}} - \frac{2}{C_1 \tilde{u}^c_y(s;s)}$

Solving the corresponding degree 2 polynomial and Using the bound on $\tilde{p}(\mathcal{S} \setminus s | s)$ by assumption, we have

$\bar{p}(\mathcal{S} \setminus s | s) \geq (1-\frac{4}{C_1} -\sqrt{\frac{8}{C_1}})\tilde{p}(\mathcal{S} \setminus s | s)$

Letting $\alpha_0 = 1-\frac{4}{C_1} -\sqrt{\frac{8}{C_1}}$, we then have:

$\bar{p}(\mathcal{S} \setminus s | s) \geq \alpha_0\tilde{p}(\mathcal{S} \setminus s | s)$

\paragraph{Step 2}
We now replace similarly $\bar{p}$ by $\tilde{p}$ for numerator of the summation in \eqref{eq:ubar}. Then we conclude by replacing $\bar{u}$ by $\tilde{u}$ using an induction proof. The extension to multiple states for $\mathcal{S}_0$ follows by summing up for each state in $\mathcal{S}_0$ and picking $C_1 = C_N = 32$ leads to the statement of the Lemma.

\mytd{We will now bound the term inside the summation in the numerator of \eqref{eq:ubar}.

We have with probability at least $1-\delta_p$:

$\bar{u}^c_y(s') \bar{p}(s|s') \leq \bar{u}^c_y(s') \cdot \paren*{\tilde{p}(s|s')  + 2\sqrt{\frac{C_1^{\delta_p}\bar{p}(s | s')}{N_k(s')}} + 2 \frac{C_2^{\delta_p}}{N_k(s')}}$

And

$\bar{u}^c_y(s') \bar{p}(s|s') \leq  \bar{u}^c_y(s')\tilde{p}(s|s')  + 2\sqrt{\bar{u}^c_y(s')\bar{p}(s | s')}\sqrt{\frac{\bar{u}^c_y(s')C_1^{\delta_p}}{N_k(s')}} + 2 \frac{\bar{u}^c_y(s')C_2^{\delta_p}}{N_k(s')}$

Replacing $\bar{u}^c_y(s')$ in the terms $\frac{\bar{u}^c_y(s')}{N_k(s')}$ and then the resulting $N_k$ by its lower bound, leads to

$\bar{u}^c_y(s') \bar{p}(s|s') \leq  \bar{u}^c_y(s')\tilde{p}(s|s')  + \sqrt{4\frac{\bar{u}^c_y(s')\bar{p}(s | s')}{C_N}} +  \frac{2}{C_N}$

Solving the resulting polynomial of degree 2 in the term $\sqrt{\bar{u}^c_y(s') \bar{p}(s|s')}$, we have:

$\bar{u}^c_y(s') \bar{p}(s|s') \leq \bar{u}^c_y(s')\tilde{p}(s|s') + \sqrt{\bar{u}^c_y(s')\tilde{p}(s|s')}\paren*{\frac{6}{C_N} + \sqrt{\frac{4}{C_N}}} + \paren*{\frac{3}{C_N} + \sqrt{\frac{1}{C_N}}}^2$.

Using $\sqrt{\bar{u}^c_y(s')\tilde{p}(s|s')} \leq \bar{u}^c_y(s')\tilde{p}(s|s') + 1$ we have

$\bar{u}^c_y(s') \bar{p}(s|s') \leq \bar{u}^c_y(s')\tilde{p}(s|s') \paren*{1 + \frac{6}{C_N} + \sqrt{\frac{4}{C_N}} } + \frac{6}{C_N} + \sqrt{\frac{4}{C_N}} + \paren*{\frac{3}{C_N} + \sqrt{\frac{1}{C_N}}}^2$

From here we have two cases:

Case 1: $\sum_{s'} \bar{u}^c_y(s') \bar{p}(s|s') \geq 1$

Letting $\alpha_1  = 1 + \frac{6}{C_N} + \sqrt{\frac{4}{C_N}} $ and $\alpha_2 = \frac{6}{C_N} + \sqrt{\frac{4}{C_N}} + \paren*{\frac{3}{C_N} + \sqrt{\frac{1}{C_N}}}^2$

We have:

\begin{align}
\sum_{s'} \bar{u}^c_y(s') \bar{p}(s|s') &\leq \sum_{s'} \bar{u}^c_y(s')\tilde{p}(s|s') \alpha_1 + \sum_{s'}\alpha_2\\
&\leq \sum_{s'} \bar{u}^c_y(s')\tilde{p}(s|s') \alpha_1 + \alpha_2 \sum_{s'} \bar{u}^c_y(s') \bar{p}(s|s')
\end{align}

And a result:

$\sum_{s'} \bar{u}^c_y(s') \bar{p}(s|s') \leq \sum_{s'} \bar{u}^c_y(s')\tilde{p}(s|s') \frac{\alpha_1}{1-\alpha_2} $

Case 2: This is simply the case where

$\sum_{s'} \bar{u}^c_y(s') \bar{p}(s|s') \leq 1$.

Combining those two cases we can conclude that

\[ \sum_{s'} \bar{u}^c_y(s') \bar{p}(s|s') \leq 1 + \sum_{s'} \bar{u}^c_y(s')\tilde{p}(s|s') \frac{\alpha_1}{1-\alpha_2}\]

\paragraph{Step 3}

Combining the previous two steps into \eqref{eq:ubar} and letting $\alpha_3 = \frac{\alpha_1}{1-\alpha_2}$, we have:

\begin{align}
\bar{u}^c_y(s) \leq  \min\curly*{y, \frac{2+ \sum_{s' \ne s}\bar{u}^c_y(s') \tilde{p}(s|s') \alpha_3}{\alpha_0 \paren*{1-\tilde{p}(s|s)}}}\label{eq:ubar_uplus}
\end{align}

\paragraph{Step 4} Replacing $\bar{u}^c_y(s')$ by $\tilde{u}^c_{y}(s'; s)$ by strong induction proof

We will perform induction on the number of states $S$

\paragraph{Hypothesis:} For any number of states $S$ and any state $s$, we have \[\bar{u}^c_y(s) \leq  \min\curly*{y, \frac{2+ \sum_{s' \ne s}\tilde{u}^c_{y}(s'; s) \tilde{p}(s|s') \alpha_3}{\alpha_0 \paren*{1-\tilde{p}(s|s)}}}\] where $\alpha_3, \alpha_0$ are some constants with $\alpha_3 \geq 1$ and $\alpha_0 \leq 1$

\paragraph{Base Case:} For $S = 2$, the hypothesis is correct. 
Consider an arbitrary $s$ and denote $\notstate$ the second state.

Condition 1:  $\bar{u}^c_y(s) \leq \tilde{u}^c_{y}(s'; s)$.
For this condition, the base case is immediately correct.

Condition 2: $\bar{u}^c_y(s) \geq \tilde{u}^c_{y}(s'; s)$

Since $\bar{u}^c_y(s) + \bar{u}^c_y(\notstate) = \tilde{u}^c_{y}(s; s) + \tilde{u}^c_{y}(\notstate; s) = y$.
So, this condition means $\bar{u}^c_y(\notstate) \leq \tilde{u}^c_{y}(\notstate; s)$. Using this relation in \eqref{eq:ubar_uplus} leads directly to the hypothesis.

\paragraph{Inductive Step}
Assume the hypothesis is true for every integers in $[2,S-1]$ for $S \geq 3$ and let's show it must also be true for $S$.

Consider an MDP $M$ with $S$ states. Let $s_0$ an arbitrary state. We will show that the hypothesis holds for $s_0$. The basic idea of the proof is that we want to construct a new MDP $M'$ with $S-1$ states by merging two states in the original MDP $M$. The new MDP $M'$ will have the property that the value of $\bar{u}^c_y(s_0)$ in $M'$ is unchanged and the value of $\tilde{u}^c_y(s_0; s_0)$ in $M'$ is much lower than in the original MDP $M$. This will allow us to upper bound the difference in the original  MDP by the difference in the new MDP.

Let $s_1$ the states with the best value (after running extended value iteration). By construction of the extended value iteration, we can conclude that $\bar{u}^c_y(s_1) \leq \tilde{u}^c_y(s_1; s_1)$ (since the optimal policy tries to go to $s_1$ as fast as possible while staying in $s_1$ as long as possible.)

Let's $\mathcal{S}_0 =  \{s_0\} \cup \{s_1\}$. Consider an MDP $M'$ where $\mathcal{S}_0$ forms a single state (And the distribution of play between the states in $\mathcal{S}_0$ follows our empirical observation during the episode). Let's denote $p(s| \mathcal{S}_0), \tilde{p}(s | \mathcal{S}_0)$ respectively the transition from the new super-state $\mathcal{S}_0$ to any other states in the true and optimistic MDP.

For any subset of next states $\mathcal{S}_1$ (it can't include separately $s_0$ or $s_1$) Let's bound the error:
$p(\mathcal{S}_1| \mathcal{S}_0) - \tilde{p}(\mathcal{S}_1 | \mathcal{S}_0)$. 

We have:

	\begin{align}
	p(\mathcal{S}_1| \mathcal{S}_0) - \tilde{p}(\mathcal{S}_1 | \mathcal{S}_0) &= \sum_{ s' \in \mathcal{S}_0  } \frac{N_k(s')}{\sum_{s'' \in \mathcal{S}_0} N_k(s'')} \paren*{p(\mathcal{S}_1| s') - \tilde{p}(\mathcal{S}_1 | s')} \\
	&\leq \sum_{ s' \in \mathcal{S}_0  } \frac{N_k(s')}{\sum_{s'' \in \mathcal{S}_0} N_k(s'')} \paren*{\sqrt{\frac{C_1^{\delta_p} p(\mathcal{S}_1| s')}{N_{t_k}(s')}} + \frac{C_2^{\delta_p}}{N_{t_k}(s')}} \\
	&= \frac{C_2^{\delta_p}}{N_k(\mathcal{S}_0)} \sum_{ s' \in \mathcal{S}_0} \frac{N_k(s')}{N_{t_k}(s')} + \sqrt{\frac{C_1^{\delta_p}}{N_k(\mathcal{S}_0)}}\sum_{ s' \in \mathcal{S}_0  } \sqrt{\frac{N_k(s')}{N_{t_k}(s')}} \sqrt{\frac{N_k(s')p(\mathcal{S}_1| s')}{N_k(\mathcal{S}_0)}} \\
	&\leq \frac{C_2^{\delta_p}}{N_k(\mathcal{S}_0)} + \sqrt{\frac{C_1^{\delta_p}}{N_k(\mathcal{S}_0)}} \sqrt{\sum_{ s' \in \mathcal{S}_0  } \frac{N_k(s')}{N_{t_k}(s')}} \sqrt{p(\mathcal{S}_1 | \mathcal{S}_0)}\\
	&\leq \frac{C_2^{\delta_p}}{N_k(\mathcal{S}_0)} + \sqrt{\frac{C_1^{\delta_p} p(\mathcal{S}_1 | \mathcal{S}_0)}{N_k(\mathcal{S}_0)}}
	\end{align}
As we can observe the upper bound follows the same format as required by \UCRLV{}. The lower bound can be established similarly. Also, showing the same for $p(\mathcal{S}_0| s) - \tilde{p}(\mathcal{S}_0 | s)$ is immediate.

So, we have an MDP $M'$ with $S-1$ states with the same requirement as the original \UCRLV{}. We can then use the inductive assumption and conclude that the hypothesis holds at $S$ too.

\paragraph{Step 5} Extension to any subsets $\mathcal{S}_0$

The extension to any subsets follow directly by summing up the previous steps for all states in $\mathcal{S}_0$.

More precisely letting $s_{t_k}$ the first state in episode $k$,
the equivalent of \eqref{eq:ubar}, \eqref{eq:utilde}, and \eqref{eq:ubar_uplus} are respectively

\begin{align}
\bar{u}^c_{y}(\mathcal{S}_0) &\leq \min \curly*{y, \frac{\Id_{s_{t_k} \in \mathcal{S}_0}}{\bar{p}(s_{t_k}|s_{t_k})} \sum_{s \in \mathcal{S}_0  } \frac{\sum_{s' \ne s}\bar{u}^c_{y}(s') \bar{p}(s|s')}{1-\bar{p}(s|s)}}\label{eq:ubar_subset}
\end{align}

\begin{align}
\tilde{u}^c_{y}(\mathcal{S}_0; s'') &\geq \min\curly*{y, \frac{1}{\tilde{p}(s''|s'')} + \sum_{s \in \mathcal{S}_0  }\frac{\sum_{s' \ne s}\tilde{u}^c_{y}(s'; s'') \tilde{p}(s|s')}{1-\tilde{p}(s|s)}}\label{eq:utilde_subset}
\end{align}

\begin{align}
\bar{u}^c_{y}(\mathcal{S}_0) &\leq \min \curly*{y, \frac{\Id_{s_{t_k} \in \mathcal{S}_0}}{\alpha_0\bar{p}(s_{t_k}|s_{t_k})} \sum_{s \in \mathcal{S}_0  } \frac{\sum_{s' \ne s}\bar{u}^c_{y}(s') \bar{p}(s|s')\alpha_3}{\alpha_0\paren*{1-\bar{p}(s|s)}}}\label{eq:ubar_uplus_subset}
\end{align}

From here, the inductive step follows by summing over all $s \in \mathcal{S}_0$ and comparing against the $s''$ giving the maximum to $\tilde{u}^c_{y}(\mathcal{S}_0; s'')$ 

\paragraph{Step 6}

Picking $C_1 = C_N = 32$ leads to the statement of the Lemma.}
\end{proof}

\begin{lemma}
\label{ucrlv:lemma:num_groups}
Let's consider the infinite set of non-overlapping intervals with non-negative endpoints $\mathcal{I}^u_{+} = \setof{]\frac{D}{2}, D], ]\frac{D}{4}, \frac{D}{2}] , ]\frac{D}{8}, \frac{D}{4}], \ldots }$ constructed in a way that the ratio between upper and lower endpoint is 2.
Given an interval $W \in  \mathcal{I}_+^u$ and any state $s$, let's $\mathcal{S}^u_W(s) $ be  the set of states $s'$ such that $u_i(s')-u_i(s) \in W$.
Let  $\mathcal{S}^{\tilde{p}}_+(s)$ contains all states $s'$ with $\tilde{p}(s'|s) - p(s'|s) > 0$. Let us define $\mathcal{S}^{\tilde{p}_+}_W(s) = \mathcal{S}^u_W(s) \cap \mathcal{S}^{\tilde{p}}_+(s)$. Let $\mathcal{I}^p = \{ ]\frac{1}{2}, 1], ]\frac{1}{4}, \frac{1}{2}], \ldots ]\frac{1}{D}, \frac{2}{D}], \ldots \}$. Given an interval $W_p \in \mathcal{I}^p$, let's call $\mathcal{S}_{W_p}$ the set of all states such that $s \in \mathcal{S}_{W_p}$ if $\tilde{p}(\mathcal{S}^{\tilde{p}_+}_W(s)|s) \in W_p$.

We have for any $W \in  \mathcal{I}_+^u$  and any $W_p \in \mathcal{I}^p$:

\[ N_k(\mathcal{S}_{W_p}) \cdot \max_{s \in \mathcal{S}_{W_p}} \tilde{p}(\mathcal{S}^{\tilde{p}_+}_W(s)|s) \cdot  \max_{s' \in \mathcal{S}^{\tilde{p}_+}_W(s)} \biggl(u_i(s') - u_i(s)\biggr) \leq 384 N_k \]
\end{lemma}
\begin{proof}
\emph{Step 1. Grouping the states in $\mathcal{S}_{W_p}$}

Let's assume that $W = ]\lowerb{W}^{u_*}, 2\lowerb{W}^{u_*}]$ for an appropriate $\lowerb{W}^{u_*}$.

We wish to group all the states $s \in \mathcal{S}_{W_p}$ into groups $\mathcal{G}_i, i\geq 1$ such that the following property is satisfied for any group $\mathcal{G}_i$.

	\begin{align}
	\min_{s \in \mathcal{G}_i }\abs*{\min_{s'' \in \left\{\cup_{s'} \mathcal{S}^{\tilde{p}_+}_W(s')| s' \in \mathcal{G}_i \right\} } u_i(s'')- u_i(s) } \geq \frac{\lowerb{W}^{u_*}}{4}.\label{eq:group_property}
	\end{align}
	
We will now show that we can create at most two groups $\mathcal{G}_1$, $\mathcal{G}_2$ satisfying property \eqref{eq:group_property} such that all the states  $s \in \mathcal{S}_{W_p}$ are assign a group.

	Let $s_m$ be the state such that \[s_m \defn \argmin_{s \in \mathcal{S}_{W_p}} \min_{s' \in  \mathcal{S}^{\tilde{p}_+}_W(s) } u_i(s').\]
	Assign $s_m$ to group $\mathcal{G}_1$. 
	At this point, $\mathcal{G}_1$ satisfies \eqref{eq:group_property} because all states $s' \in \mathcal{S}^{\tilde{p}_+}_W(s_m)$ satisfy $u_i(s') - u_i(s_m) \in W$ by construction. 
		
	By definition of $s_m$, adding any other states $ s \in \mathcal{S}_{W_p}$ to $\mathcal{G}_1$ would not change the inner-minimum of \eqref{eq:group_property}. As a result, we satisfy \eqref{eq:group_property} by adding  any state $ s \in \mathcal{S}_{W_p}$ to $\mathcal{G}_1$ such that \[\abs{u_i(s) - \min_{s' \in \mathcal{S}^{\tilde{p}_+}_W(s_m) } u_i(s')} \geq \frac{\lowerb{W}^{u_*}}{4}.\]
	Thus, the states in $\mathcal{G}_1$ satisfy \eqref{eq:group_property}.

	All the remaining state $s \in \mathcal{S}_{W_p}$ satisfy \[\min_{s' \in \mathcal{S}^{\tilde{p}_+}_W(s_m) } u_i(s') - \frac{\lowerb{W}^{u_*}}{4} \leq u_i(s) \leq \min_{s' \in \mathcal{S}^{\tilde{p}_+}_W(s_m) } u_i(s') + \frac{\lowerb{W}^{u_*}}{4}.\]
	Let's assign all those states to $\mathcal{G}_2$. Now, we show that $\mathcal{G}_2$ also satisfy \eqref{eq:group_property}.
	
	 For any state $s$ in group $\mathcal{G}_2$, the corresponding value $u_i(s)$ satisfies
	 \begin{align}
	 u_i(s)  \in \left]\min_{s' \in \mathcal{S}^{\tilde{p}_+}_W(s_m) } u_i(s') - \frac{\lowerb{W}^{u_*}}{4}, \min_{s' \in \mathcal{S}^{\tilde{p}_+}_W(s_m) } u_i(s') + \frac{\lowerb{W}^{u_*}}{4}\right[.\label{ucrlv:eq:born1}
	 \end{align}
	 
	 We also know that all these states  in group $\mathcal{G}_2$ have values in same interval $W = ]\lowerb{W}^{u_*}, 2\lowerb{W}^{u_*}]$.
	 Hence, for states in group $\mathcal{G}_2$, the values of the inner minima $\min_{s'' \in \left\{\cup_{s'} \mathcal{S}^{\tilde{p}_+}_W(s')| s' \in \mathcal{G}_2 \right\} } u_i(s'')$ belongs to the interval:
	 
	 \begin{align}
	  \left]\min_{s' \in \mathcal{S}^{\tilde{p}_+}_W(s_m) } u_i(s') - \frac{\lowerb{W}^{u_*}}{4} + \lowerb{W}^{u_*} , \min_{s' \in \mathcal{S}^{\tilde{p}_+}_W(s_m) } u_i(s') + \frac{\lowerb{W}^{u_*}}{4} + 2\lowerb{W}^{u_*}\right[.\label{ucrlv:eq:born2}
	 \end{align}
	 
	 Now, checking all four possible combinations of differences between the endpoints of the intervals in \eqref{ucrlv:eq:born1} and \eqref{ucrlv:eq:born2} show that  states in group $\mathcal{G}_2$ also satisfy \eqref{eq:group_property}.
	 
\paragraph{Step 2: Calculations}

Letting $\mathcal{S}^1_i = \cup_{s'} \mathcal{S}^{\tilde{p}_+}_W(s')| s' \in \mathcal{G}_i$,
$y_i = \min\{D, \frac{1}{\min_{s \in \mathcal{G}_i}\tilde{p}(\mathcal{S}^1_i|s)}\}$ and 
 $s_i^* = \argmax_{s \in \mathcal{G}_i} \tilde{u}^c_{y_i}(\mathcal{G}_i | s)$We have:

\begin{align}
N' &= N_k(\mathcal{S}_{W_p}) \cdot \max_{s \in \mathcal{S}_{W_p}}\paren*{ \tilde{p}(\mathcal{S}^{\tilde{p}_+}_W(s)|s) \cdot  \max_{s' \in \mathcal{S}^{\tilde{p}_+}_W(s)} \biggl(u_i(s') - u_i(s)\biggr)}\\
&= \sum_{i \in \{1,2\}} N_k(\mathcal{G}_i) \cdot \max_{s \in \mathcal{S}_{W_p}}\paren*{ \tilde{p}(\mathcal{S}^{\tilde{p}_+}_W(s)|s) \cdot  \max_{s' \in \mathcal{S}^{\tilde{p}_+}_W(s)} \biggl(u_i(s') - u_i(s)\biggr)}\\
&\leq \sum_{i \in \{1,2\}} N_k(\mathcal{G}_i) \cdot \max_{s \in \mathcal{S}_{W_p}}\paren*{ \tilde{p}(\mathcal{S}^{\tilde{p}_+}_W(s)|s)} \cdot \max_{s \in \mathcal{S}_{W_p}}\paren*{\max_{s' \in \mathcal{S}^{\tilde{p}_+}_W(s)} \biggl(u_i(s') - u_i(s)\biggr)}\\
&\leq 4 \sum_{i \in \{1,2\}} N_k(\mathcal{G}_i) \cdot \min_{s \in \mathcal{G}_i}\paren*{ \tilde{p}(\mathcal{S}^{\tilde{p}_+}_W(s)|s)} \cdot \min_{s \in \mathcal{G}_i}\paren*{\min_{s' \in \mathcal{S}^{\tilde{p}_+}_W(s)} \biggl(u_i(s') - u_i(s)\biggr)}\\
&\leq 4 \sum_{i \in \{1,2\}} N_k(\mathcal{G}_i) \cdot  \min_{s \in \mathcal{G}_i}\paren*{\tilde{p}(\mathcal{S}^1_i|s)} \cdot \min_{s \in \mathcal{G}_i}\paren*{\min_{s' \in \mathcal{S}^{\tilde{p}_+}_W(s)} \biggl(u_i(s') - u_i(s)\biggr)}\label{ucrlv:eq:Nprime_subset}\\
&\leq 16 \sum_{i \in \{1,2\}} N_k(\mathcal{G}_i) \cdot  \min_{s \in \mathcal{G}_i}\paren*{\tilde{p}(\mathcal{S}^1_i|s)} \cdot \min_{s \in \mathcal{G}_i}\paren*{\abs*{\min_{s'' \in \mathcal{S}^1_i } u_i(s'')- u_i(s) }}\label{ucrlv:eq:Nprime_apply_group_property}\\
&\leq 16 \sum_{i \in \{1,2\}} N_k(\mathcal{G}_i) \cdot  \min_{s \in \mathcal{G}_i}\paren*{\tilde{p}(\mathcal{S}^1_i|s)} \cdot \paren*{\abs*{\min_{s'' \in \mathcal{S}^1_i } u_i(s'')- u_i(s_i^*) }}\\
&\leq 192 \sum_{i \in \{1,2\}} \frac{N_k}{y_i} \tilde{u}^c_{y_i}(\mathcal{G}_i | s_i^*) \cdot  \min_{s \in \mathcal{G}_i}\paren*{\tilde{p}(\mathcal{S}^1_i|s)} \cdot \paren*{\abs*{\min_{s'' \in \mathcal{S}^1_i } u_i(s'')- u_i(s_i^*) }}\label{ucrlv:eq:apply_lemma_actual_visit_true_visit}\\
&\leq 192 \sum_{i \in \{1,2\}} N_k\label{ucrlv:eq:apply_lemma_link_value_counts}\\
&= 384 N_k
\end{align}
\eqref{ucrlv:eq:Nprime_subset} comes from the fact that $ \mathcal{S}^{\tilde{p}_+}_W(s) \in \mathcal{S}^1_i$.
\eqref{ucrlv:eq:Nprime_apply_group_property} comes from the fact that the states in group $\mathcal{G}_i$ satisfy property \eqref{eq:group_property}.

\eqref{ucrlv:eq:apply_lemma_actual_visit_true_visit} comes by replacing $N_k(\mathcal{G}_i)$ using Lemma \ref{lemma_bound_on_visits} with $y = y_i$.
\eqref{ucrlv:eq:apply_lemma_link_value_counts} comes by applying Lemma \ref{lemma_probs_vs_values_sum} with $\mathcal{S}_0 = \mathcal{G}_{i}$ and $\mathcal{S}_1 = \mathcal{S}_i^1$.

\end{proof}

\section{The Effect of Extended Doubling Trick}
\begin{reptheorem}{thm:episodes}[Bounding the number of episodes]
	The number of episodes $m$ is upper bounded by
	\[m \leq SA \log_2\paren*{\frac{8T}{SA}}\]
\end{reptheorem}

\begin{proof}
The main difference between our \emph{extended doubling trick} and the standard \cite{jaksch2010near} is that we are not guaranteed to double any single state for any given \episode{}. As a result, the number of \episodes{} could be arbitrarily large. lemma \ref{thm:episodes} proves that this is not the case. The main intuition is: since the average number of states doubled per \episode{} is 1, then after $SA$ \episodes{} we can be sure to have doubled some states $SA$ times.

For each $(s,a)$ we would to list a set $\mathcal{K}(s,a)$ of \episodes{} indices where $(s,a)$ has been doubled between two consecutive index. More formally, let $\mathcal{K}(s,a) = \{k_1(s,a), k_2(s,a) \ldots k_{\size{\mathcal{K}(s,a)}}\}$ a list of \episodes{} number such that for all $i \geq 1$:

\begin{align}
\sum_{k=k_i(s,a)}^{k_{i+1}(s,a)-1}  \frac{ N_k(s,a)}{N_{t_{k_i}}(s,a)} &\leq 1\\
\sum_{k=k_i(s,a)+1}^{k_{i+1}(s,a)}  \frac{ N_k(s,a)}{N_{t_{k_i}}(s,a)} &\geq 1\label{ucrlv:eq:db_trick}\\
N_{t_{k_i}}(s,a) & > 0\\
k_i(s,a) < k_{i+1}(s,a)
\end{align}

We will now relate the total number of episodes to each $\mathcal{K}(s,a)$.

Since by construction we know that $\sum_{s,a} \frac{N_k(s,a)}{\max\{1, N_{t_k}(s,a)\}} > 1$, we have:
\begin{align}
m &\leq \sum_{k=1}^{m} \sum_{s,a} \frac{N_k(s,a)}{\max\{1, N_{t_k}(s,a)\}}\\
&= \sum_{s,a} \paren*{\sum_{k=1}^{k_1(s,a)-1} \frac{N_k(s,a)}{\max\{1, N_{t_k}(s,a)\}} + \sum_{k=k_1(s,a)}^{m} \frac{N_k(s,a)}{N_{t_k}(s,a)}}\\
&= \sum_{s,a}\paren*{ 1 + \sum_{i =1}^{\size{\mathcal{K}(s,a)}} \sum_{k=k_i(s,a)}^{k_{i+1}(s,a)-1} \frac{N_k(s,a)}{N_{t_k}(s,a)} + \sum_{k=k_{\size{\mathcal{K}(s,a)}}(s,a)}^m \frac{N_k(s,a)}{N_{t_k}(s,a)} }\\
&\leq \sum_{s,a}\paren*{ 1 + \size{\mathcal{K}(s,a)} + 1}\\
& = 2SA +  \sum_{s,a}\size{\mathcal{K}(s,a)}\label{eq:m_bound}
\end{align}

Now noting that for any two consecutive $i, i+1$, we have $N_{t_{k_{i+1}}(s,a)} \geq 2 N_{t_{k_{i}}(s,a)}$ and denoting $N(s,a)$ the total number of times $(s,a)$ is played; we have
\begin{align}
N(s,a) &= \sum_{k=1}^m N_k(s,a)\\
&= \sum_{k=1}^{k_1(s,a)} N_k(s,a) + \sum_{i=1}^{\size{\mathcal{K}(s,a)}}\sum_{k=k_i(s,a)+1}^{k_{i+1}(s,a)} N_k(s,a) + \sum_{k=k_{\size{\mathcal{K}(s,a)}}(s,a) + 1}^{m} N_k(s,a)\\
&\geq \sum_{i = 1}^{\size{\mathcal{K}(s,a)}} N_{t_{k_i}}(s,a)\label{ucrlv:use:eq:db_trick}\\
&\geq \sum_{i = 1}^{\size{\mathcal{K}(s,a)}} 2^{i-1}\\
&= 2^{\size{\mathcal{K}(s,a)}} - 1\label{eq:T_bound}
\end{align}

\Cref{ucrlv:use:eq:db_trick} comes by using \cref{ucrlv:eq:db_trick}.

\eqref{eq:m_bound} implies that:$\sum_{s,a} \size{\mathcal{K}(s,a)} \geq m - 2SA$ and 
$\sum_{s,a} 2^{\size{\mathcal{K}(s,a)}} \geq SA 2^{\sum_{s,a} \size{\mathcal{K}(s,a)}/(SA)} \geq SA 2^{\frac{m-2SA}{SA}}$

Which together with \eqref{eq:T_bound} implies:

\[ T \geq SA (2^{\frac{m-2SA}{SA}} - 1)\]

Which leads to $m \leq SA \log_2(\frac{T}{SA} + 1) + 2SA$
and the lemma follows for $T \geq SA$.
\end{proof}

\section{Technical Lemmas for Convergence of Extended Value Iteration and Its Consequences}
\label{sec:app_evi}
In this section, we proved fundamental results related to the modified extended value iteration.

\begin{reptheorem}{thm:evi_convergence}[Convergence of Extended Value Iteration]
	Let $\mathcal{M}$ be the set of all MDPs with state space $\mathcal{S}$, action space $\mathcal{A}$, transitions probabilities $\tilde{p}(s,a)$ and mean rewards $\tilde{r}(s,a)$ that satisfy \eqref{eq:plausible_rewards} and \eqref{eq:plausible_transitions} for given probabilities distribution $\bar{p}(s,a)$, $\bar{r}(s,a)$ in $[0,1]$. If $\mathcal{M}$ contains at least one communicating MDP, extended value iteration in Algorithm \ref{algo:extended_vi} converges. Further, stopping extended value iteration when:
	\[\max_{s}\{u_{i+1}(s)-u_i(s)\}-\min_{s}\{u_{i+1}(s)-u_i(s)\} \leq \epsilon ,\]
	the greedy policy $\pi$ with respect to $u_i$ is $\epsilon$-optimal meaning $V(\pi) \geq V^*_{\mathcal{M}} - \epsilon$
	and $$\abs{u_{i+1}(s)-u_i(s)- V(\pi)} \leq \epsilon.$$
\end{reptheorem}
\begin{proof}
	Using Corollary \ref{lemma:optimism_p} and Lemma \ref{lemma:trans}, we can observe that Algorithm \ref{algo:extended_vi} computes correctly the maximum in \eqref{eq:extended_vi}. Since by assumption, the extended MDP is communicating and Algorithm \ref{algo:extended_vi} always chooses policies with aperiodic transition matrix (See discussion in Section 3.1.3 of \cite{jaksch2010near}), we can conclude that Theorem 9.4.4 and 9.4.5 of \cite{puterman2014markov} holds which lead to the first statement.

	The last statement is a direct consequence of the first from Theorem 8.5.6 of \cite{puterman2014markov}.
\end{proof}

\begin{lemma}
\label{fact:opt_p}
Consider an ordering of the states such that $u_i(s'_1) \geq u_i(s'_2) \ldots \geq u_i(s'_S)$.
For any model $M$ with transitions $p$ such that $M \in \mathcal{M}_k$;  any state-action $(s,a)$, the transition $\tilde{p}$ returned by \Call{OptimisticTransition}{} in Algorithm \ref{algo:extended_vi} satisfies for  any $l \leq S$:
\[\sum_{j=1}^l p(s'_j | s, a)  \leq \sum_{j=1}^l \tilde{p}(s'_j | s, a).\]
\end{lemma}
\begin{proof}
Recall that for any $l$, \[\opt{p}(s'_l | s,a) \gets \min\left\{ \upperb{p}(\mathcal{S}_1^l | s,a)-\opt{p}(\mathcal{S}_1^{l-1} | s,a),  1- \opt{p}(\mathcal{S}_{1}^{l-1} | s, a) \right\}\]

If the minimum is the second term, then we have
\begin{align}
\tilde{p}(\mathcal{S}_1^{l} | s,a) &=\sum_{j=1}^l \tilde{p}(s'_j | s, a)\\
&= \tilde{p}(\mathcal{S}_1^{l-1} | s,a) + \tilde{p}(s'_l | s,a)\\
&= \tilde{p}(\mathcal{S}_1^{l-1} | s,a) + 1- \tilde{p}(\mathcal{S}_1^{l-1} | s,a)\\
&= 1\\
&\geq \sum_{j=1}^l p(s'_j | s, a)
\end{align}

Now let's assume that the minimum is the first term

\begin{align}
\tilde{p}(\mathcal{S}_1^{l} | s,a) &=\sum_{j=1}^l \tilde{p}(s'_j | s, a)\\
&= \tilde{p}(\mathcal{S}_1^{l-1} | s,a) + \tilde{p}(s'_l | s,a)\\
&= \tilde{p}(\mathcal{S}_1^{l-1} | s,a) + \upperb{p}(\mathcal{S}_1^l | s,a)-\opt{p}(\mathcal{S}_1^{l-1} | s,a)\\
&= \upperb{p}(\mathcal{S}_1^l | s,a) \label{ucrlv:eq:assumption_plausible_apply}\\
&\geq \sum_{j=1}^l p(s'_j | s, a)
\end{align}
\Cref{ucrlv:eq:assumption_plausible_apply} comes from the fact that $M$ is assumed to be in the plausible set $\mathcal{M}_k$.

This proves \Cref{fact:opt_p}.
\end{proof}

\begin{lemma}\label{lemma:opt_p_u}
Consider an ordering of the states such that $u_i(s'_1) \geq u_i(s'_2) \ldots \geq u_i(s'_S)$.
For any model $M$ with transitions $p$ such that $M \in \mathcal{M}_k$;  any state-action $(s,a)$, the transition $\tilde{p}$ returned by \Call{OptimisticTransition}{} (Algorithm \ref{algo:extended_vi}) satisfies for  any $l \leq S$:
\[ \sum_{j=1}^l \tilde{p}(s'_j|s,a) u_i(s'_j) - \sum_{j=1}^l p(s'_j|s,a) u_i(s'_j) \geq \paren*{ \sum_{j=1}^l \tilde{p}(s'_j|s,a) - \sum_{j=1}^l p(s'_j|s,a)} \min_{j\leq l} u_i(s'_j) \]
\end{lemma}
\begin{proof}
We prove the statement by induction on $l$ for any $s,a$ and as a result removes dependency of $p, \tilde{p}$ on $s,a$ in this proof.

\textbf{Base Case:} For $l=1$, the statement is true since $\tilde{p}(s'_1)u_i(s'_1)-p(s'_1 )u_i(s'_1) = (\tilde{p}(s'_1) - p(s'_1 ))u_i(s'_1)$

\textbf{Inductive Case:} Assume that the statement is true up to $l$. Now we need to show it also holds at $l+1$.

Let $\tilde{v}_{l+1} = \sum_{j=1}^{l+1} \tilde{p}(s'_j) u_i(s'_j)$ and $v_{l+1} =  \sum_{j=1}^l p(s'_j) u_i(s'_j)$
We have:

\begin{align}
\tilde{v}_{l+1} - v_{l+1} &=\sum_{j=1}^{l+1} \tilde{p}(s'_j) u_i(s'_j) - \sum_{j=1}^{l+1} p(s'_j) u_i(s'_j)\\
&= \sum_{j=1}^{l} \tilde{p}(s'_j) u_i(s'_j) - \sum_{j=1}^{l} p(s'_j) u_i(s'_j) + (\tilde{p}(s'_{j+1}) - p(s'_{j+1})) u_i(s'_{j+1})\label{eq:72}\\
& \geq \paren*{ \sum_{j=1}^l \tilde{p}(s'_j|s,a) - \sum_{j=1}^l p(s'_j|s,a)} \min_{j\leq l} u_i(s'_j) + (\tilde{p}(s'_{j+1}) - p(s'_{j+1})) u_i(s'_{j+1})\label{eq:73}\\
& \geq \paren*{ \sum_{j=1}^l \tilde{p}(s'_j|s,a) - \sum_{j=1}^l p(s'_j|s,a)} u_i(s'_{j+1})  + (\tilde{p}(s'_{j+1}) - p(s'_{j+1})) u_i(s'_{j+1})\label{eq:74}\\
&=\sum_{j=1}^{l+1} \tilde{p}(s'_j|s,a) u_i(s'_j) - \sum_{j=1}^{l+1} p(s'_j|s,a) u_i(s'_j)\label{eq:75} \\
&= \paren*{ \sum_{j=1}^{l+1} \tilde{p}(s'_j|s,a) - \sum_{j=1}^{l+1} p(s'_j|s,a)} \min_{j\leq l} u_i(s'_j)\label{eq:76}
\end{align}
\eqref{eq:73} comes by the inductive case.
\eqref{eq:74} comes because by Lemma \ref{fact:opt_p}, the difference in the $p$ is positive. Also, the states were sorted in descending order based on $u_i$.
\eqref{eq:76} comes from the fact that the states were sorted in descending order based on $u_i$.

which concludes the proof.
\end{proof}
\begin{corollary}
\label{lemma:optimism_p}
Consider an ordering of the states such that $u_i(s'_1) \geq u_i(s'_2) \ldots \geq u_i(s'_S)$.
For any model $M$ with transitions $p$ such that $M \in \mathcal{M}_k$;  any state-action $(s,a)$, the transition $\tilde{p}$ returned by Algorithm \ref{algo:extended_vi} satisfies:

\[ \sum_{j=1}^S \tilde{p}(s'_j|s,a) u_i(s'_j) - \sum_{j=1}^S p(s'_j|s,a) u_i(s'_j) \geq 0 \]

\end{corollary}
\begin{proof}
Immediate by Lemma \ref{lemma:opt_p_u} with $l = S$.
\end{proof}

\begin{lemma}
\label{lemma:trans}
For all state-action pairs $(s,a)$, if the function $\mathcal{S}_c \to \upperb{p}(\mathcal{S}_c|s,a)$ is submodular,
then the transitions $\tilde{p}(. | s,a)$ returned by the function \Call{OptimisticTransition}{} (Algorithm \ref{algo:extended_vi}) satisfy:

\begin{equation}
\tilde{p}(\mathcal{S}_c|s,a) \leq \upperb{p}(\mathcal{S}_c|s,a) \; \forall  \mathcal{S}_c \subseteq \mathcal{S} \label{ucrlv:eq:constraints}
\end{equation}

\end{lemma}
\begin{proof}
The following proof is done for any $(s,a)$. For simplicity, to designate probabilities given $s,a$, we omit the dependency on $s,a$ and $\delta_p^k$. Specifically, we write $\opt{p}(s'_j | s,a) = \opt{p}(s'_j)$ for this proof.

Recall that by construction, Function \Call{OptimisticTransition}{} (Algorithm \ref{algo:extended_vi}) sorts the states in a given order $s'_1, s'_2, \ldots s'_S$ and greedily assign  $\opt{p}(s'_l | s,a)$ as:

\[\opt{p}(s'_j | s,a) \gets \min\left\{ \upperb{p}(\mathcal{S}_1^j | s,a)-\opt{p}(\mathcal{S}_1^{j-1} | s,a),  1- \opt{p}(\mathcal{S}_{1}^{j-1} | s, a) \right\}.\]
with $\mathcal{S}_1^j = \curly*{s'_1, \ldots s'_j}$


We will prove the statement of the lemma by induction on $j$.

\textbf{Base Case:} For $j=1$, condition \eqref{ucrlv:eq:constraints} holds for all possible subset of states of $\mathcal{S}_1^1$ since by construction $\tilde{p}(s'_1) = \min\curly*{\upperb{p}(s'_1), 1} \leq \upperb{p}(s'_1)$.

\textbf{Inductive step:} We assume that
for any subset $\mathcal{S}_c$ of $\mathcal{S}_1^j$, we have: $\tilde{p}(\mathcal{S}_c) \leq \upperb{p}(\mathcal{S}_c)$

Now we need to prove that the inductive assumption also holds for $j+1$. That is:
$\tilde{p}(\mathcal{S}_c) \leq \upperb{p}(\mathcal{S}_c)\; \forall \mathcal{S}_c \subseteq \mathcal{S}_1^{j+1}$.

For that, we just need to show that the inductive assumption holds for all subset of $\mathcal{S}_1^{j+1}$ that contains state $s'_{j+1}$. Consider any subset of states $\mathcal{S}_0 \subseteq \mathcal{S}_1^{j}$ . 
%

We have:

\begin{align}
\tilde{p}(\mathcal{S}_0 \cup \curly*{s'_{j+1}}) &= \tilde{p}(\mathcal{S}_0) + \tilde{p}( \curly*{s'_{j+1}})\\
&\leq \tilde{p}(\mathcal{S}_0) + \paren*{\upperb{p}(\mathcal{S}_1^{j+1})-\tilde{p}(\mathcal{S}_1^j)}\\
&\leq  \tilde{p}(\mathcal{S}_0) + \paren*{\upperb{p}(\mathcal{S}_0 \cup \curly{s'_{j+1}})-\tilde{p}(\mathcal{S}_0)}\label{ucrlv:apply_submodularity2}\\
&=\upperb{p}(\mathcal{S}_0 \cup \curly{s'_{j+1}})\label{ucrlv:eq:conclude_sub}
\end{align}
\Cref{ucrlv:apply_submodularity2} comes directly due to the submodularity of the function $\mathcal{S}_c \to \upperb{p}(\mathcal{S}_c) - \tilde{p}(\mathcal{S}_c)$ and the fact that $\mathcal{S}_0 \subseteq \mathcal{S}_1^j$. Indeed, $\mathcal{S}_c \to -\tilde{p}(\mathcal{S}_c)$ is submodular since for any $X, Y, x : X \subseteq Y, x \in X$ we have $-\tilde{p}(Y)- (-\tilde{p}(Y\setminus x)) = - \tilde{p}(x) = -\tilde{p}(X)- (-\tilde{p}(X\setminus x))$. Also, $\mathcal{S}_c \to \upperb{p}(\mathcal{S}_c)$ is submodular (See \ref{req:submodular}). Using \Cref{ucrlv:theorem_modular_summation} that shows that the sum of two submodular functions is submodular, we can conclude that $\mathcal{S}_c \to \upperb{p}(\mathcal{S}_c) - \tilde{p}(\mathcal{S}_c)$ is submodular.

\Cref{ucrlv:eq:conclude_sub} means that the inductive statement is also true for $\mathcal{S}_1^{j+1}$ concluding the induction proof.



This means that the inductive statement is also true for all subsets up to step $j+1$ concluding the induction proof.
\end{proof}

\begin{lemma}\label{lemma:opt_diameter}
	Since the set of MDPs $\mathcal{M}$ in extended value iteration contains at least one communicating MDP of diameter $D$,
	$$\max_s u_i(s) - \min_s u_i(s) \leq D.$$
\end{lemma}
\begin{proof}
	This lemma is a direct consequence of Equation 11 in Section 4.3.1 in \citep{jaksch2010near}.
\end{proof}

\begin{lemma}\label{lemma:sum_p_vs_tilde}
	For any state $s$ such that $\tilde{\Delta}_k^p(s) = N_k(s) \cdot\sum_{s'} p(s' | s)\paren*{u_i(s')-u_i(s)} \geq 0$, we have:
	\[ \sum_{s'} p(s' | s)\abs*{u_i(s)-u_i(s')} \leq 2\sum_{s'} \tilde{p}(s' | s)\abs*{u_i(s)-u_i(s')}\]
\end{lemma}
\begin{proof}

Since by assumption, $\sum_{s'} p(s' | s)\paren*{u_i(s')-u_i(s)} \geq 0$, we have that

\[ \sum_{s' \in \mathcal{S}} p(s' | s)\abs*{u_i(s')-u_i(s)} \leq \sum_{s': u_i(s')-u_i(s) \geq 0} 2p(s' | s)\abs*{u_i(s')-u_i(s)}.\]

Using Lemma \ref{lemma:opt_p_u} for an $l$ up to the last state with $u_i(s')-u_i(s) \geq 0$ proves the statement of this lemma.
%
\end{proof}

\section{Useful Existing Definitions and Results}
\begin{definition}[Monotone Set Function]
\label{ucrlv:definition_monotone}
Let $\Omega$ a finite set. A set function defined as $f: 2^{\Omega} \to \mathbb{R}$ is \emph{monotone} if for any $X \subseteq Y$ $f(X) \leq f(Y)$.
\end{definition}
\begin{definition}[Submodular and Supermodular Function~\citep{schrijver2003combinatorial}]
\label{ucrlv:definition_submodular}
	Let $\Omega$ a finite set. A submodular function is a set function $f: 2^{\Omega} \to \mathbb{R}$
	which satisfies the following condition:
	
	For every $X,Y \subseteq \Omega$, $X \subseteq Y$ and every $x \in X$ we have: $f(Y) - f(Y\setminus {x}) \leq f(X) - f(X\setminus {x})$
	where $2^{\Omega}$ is the set of all subsets of $\Omega$. 
	
	\begin{itemize}
	\item A function $f$ is supermodular if $-f$ is submodular.
	\item A function $f$ is modular if it is both supermodular and submodular (i.e the submodularity condition is a strict equality)
	
	\end{itemize}
	
\end{definition}

\begin{theorem}[Concave composing monotonic modular is submodular \citep{krause2014submodular,yu2015submodular}]
\label{ucrlv:theorem_modular_concave_composition}
Let $\Omega$ a finite set with $\size{\Omega} \geq 3$, $M: 2^{\Omega} \to \mathbb{R}$ and $g: \mathbb{R} \to \mathbb{R}$. $F \defn g \circ M$ is (strictly) submodular for every monotone modular $M$ if and only if $g$ is (stricly) concave.
\end{theorem}

\begin{theorem}[Summation preserves submodularity \citep{krause2014submodular,yu2015submodular}]
\label{ucrlv:theorem_modular_summation}
If $f$ and $g$ are two submodular functions, then $f + g$ is submodular.
\end{theorem}

\begin{lemma}[Union Bound (or Boole's Inequality)]
	\label{union_bound}
	For a countable set of events $A_1, A_2, \ldots$ we have:
	\[\Prob\left(\bigcup_i A_i\right) \leq \sum_{i} \Prob(A_i)\]
\end{lemma}

\begin{theorem}[Empirical Bernstein Inequality \citep{empirical_bernstein_bounds}]
	\label{theo:empirical_bernstein}
	Let $Z, Z_1, Z_2, \ldots Z_n$ be i.i.d random variables with values in $[0,1 ]$, common mean $\E Z$ and let $\delta > 0$. Then, we have:	
	\[ \Prob \left(\abs*{\E Z - \frac{1}{n} \sum_{i=1}^{n} Z_i} \geq \sqrt{\frac{2\Var_n(\bm{Z})\ln 2/\delta}{n}} + \frac{7\ln 2/\delta}{3(n-1)}\right) \leq 2\delta\]	
%
	where $\bm{Z} = \paren*{Z_1, \ldots Z_n}$, $V_n(\bm{Z})$ is the sample variance: \[\Var_n(\bm{Z}) = \frac{1}{n(n-1)} \sum_{1\leq i < j \leq n} (Z_i-Z_j)^2\]
	
\end{theorem}

\begin{theorem}[Bennett's inequality \citep{empirical_bernstein_bounds}]
	\label{theo:non_empirical_bernstein}
	Under the conditions of \Cref{theo:empirical_bernstein}, we have:
		
	\[ \Prob \left(\abs*{\E Z - \frac{1}{n} \sum_{i=1}^{n} Z_i} \geq \sqrt{\frac{2\Var(Z)\ln 1/\delta}{n}} + \frac{\ln 1/\delta}{3n}\right) \leq 2\delta\]	
%
	where $\Var(Z)$ is the variance $\Var(Z) = \E\paren*{Z - \E Z}^2$.
	
\end{theorem}

\begin{theorem}[Bound on the Sample Variance \citep{empirical_bernstein_bounds}]
	\label{theo:sample_variance_bound}
	Let $n \geq 2$ and $\bm{Z} = \paren*{Z_1, \ldots Z_n}$ be a vector of independent random variables with values in $[0,1]$. Then for $\delta > 0$ and writing $\E \Var_n$ for $\E_{\bm{Z}} \Var_n(\bm{Z})$ with $\Var_n(\bm{Z})$ defined as in Theorem \ref{theo:empirical_bernstein},
	\[\Prob\left(\abs{\sqrt{\Var_n(\bm{Z})} - \sqrt{\E \Var_n}} > \sqrt{\frac{2 \ln 1/\delta}{n-1}}\right) \leq \delta \]
\end{theorem}

\bibliography{../rlfastposteriorsampling}  

\pagebreak
\appendix
\onecolumn

\end{document}